\documentclass{article}

\usepackage{microtype}
\usepackage{graphicx}
\usepackage{subfigure}
\usepackage{booktabs} 
\usepackage{amsmath}
\usepackage{amsthm}
\usepackage{amssymb}
\usepackage{mathtools}

\usepackage{hyperref}

\usepackage{scalerel,stackengine}
\stackMath
\newcommand\reallywidehat[1]{%
\savestack{\tmpbox}{\stretchto{%
  \scaleto{%
    \scalerel*[\widthof{\ensuremath{#1}}]{\kern-.6pt\bigwedge\kern-.6pt}%
    {\rule[-\textheight/2]{1ex}{\textheight}}
  }{\textheight}%
}{0.5ex}}%
\stackon[1pt]{#1}{\tmpbox}%
}

\DeclareMathOperator*{\argmin}{arg\,min}
\DeclareMathOperator*{\argmax}{arg\,max}
\newtheorem{thm}{Theorem}
\newtheorem{lemma}[thm]{Lemma}
\newtheorem*{remark}{Remark}
\newcommand{\eqn}[1]{Eqn.~(\ref{#1})}
\newcommand{\sect}[1]{Section~\ref{#1}}

\newcommand{\fig}[1]{Figure~\ref{#1}}
\newcommand{\alg}[1]{Alg.~\ref{#1}}
\newcommand{\apdx}[1]{Appendix~\ref{#1}}

\newcommand{\eg}{e.g.~}

\usepackage[accepted]{icml2018}

\icmltitlerunning{Disentangling by Factorising}

\begin{document}

\twocolumn[
\icmltitle{Disentangling by Factorising}




\icmlsetsymbol{equal}{*}

\begin{icmlauthorlist}
\icmlauthor{Hyunjik Kim}{dm,ox}
\icmlauthor{Andriy Mnih}{dm}
\end{icmlauthorlist}

\icmlaffiliation{ox}{Department of Statistics, University of Oxford}
\icmlaffiliation{dm}{DeepMind, UK}

\icmlcorrespondingauthor{Hyunjik Kim}{hyunjikk@google.com}




\icmlkeywords{Disentangling, Representation Learning, Generative Models, Machine Learning, ICML}

\vskip 0.3in
]



\printAffiliationsAndNotice{}  

\begin{abstract}
We define and address the problem of unsupervised learning of disentangled representations on data generated from independent factors of variation. 
We propose FactorVAE, a method that disentangles by encouraging the distribution of representations to be factorial and hence independent across the dimensions. We show that it improves upon $\beta$-VAE by providing a better trade-off between disentanglement and reconstruction quality. Moreover, we highlight the problems of a commonly used disentanglement metric and introduce a new metric that does not suffer from them.
\end{abstract}

\section{Introduction}
\label{intro}

Learning interpretable representations of data that expose semantic meaning has important consequences for artificial intelligence. Such representations are useful not only for standard downstream tasks such as supervised learning and reinforcement learning, but also for tasks such as transfer learning and zero-shot learning where humans excel but machines struggle \cite{lake2016building}. There have been multiple efforts in the deep learning community towards learning factors of variation in the data, commonly referred to as learning a \textit{disentangled representation}. While there is no canonical definition for this term, we adopt the one due to \citet{bengio2013representation}: a representation where a change in one dimension corresponds to a change in one factor of variation, while being relatively invariant to changes in other factors. In particular, we assume that the data has been generated from a fixed number of independent factors of variation.\footnote[3]{We discuss the limitations of this assumption in \sect{metric}.} We focus on image data, where the effect of factors of variation is easy to visualise.

\begin{figure}[h!]
\vskip -0.1in
  \centering
  \includegraphics[width=\columnwidth]{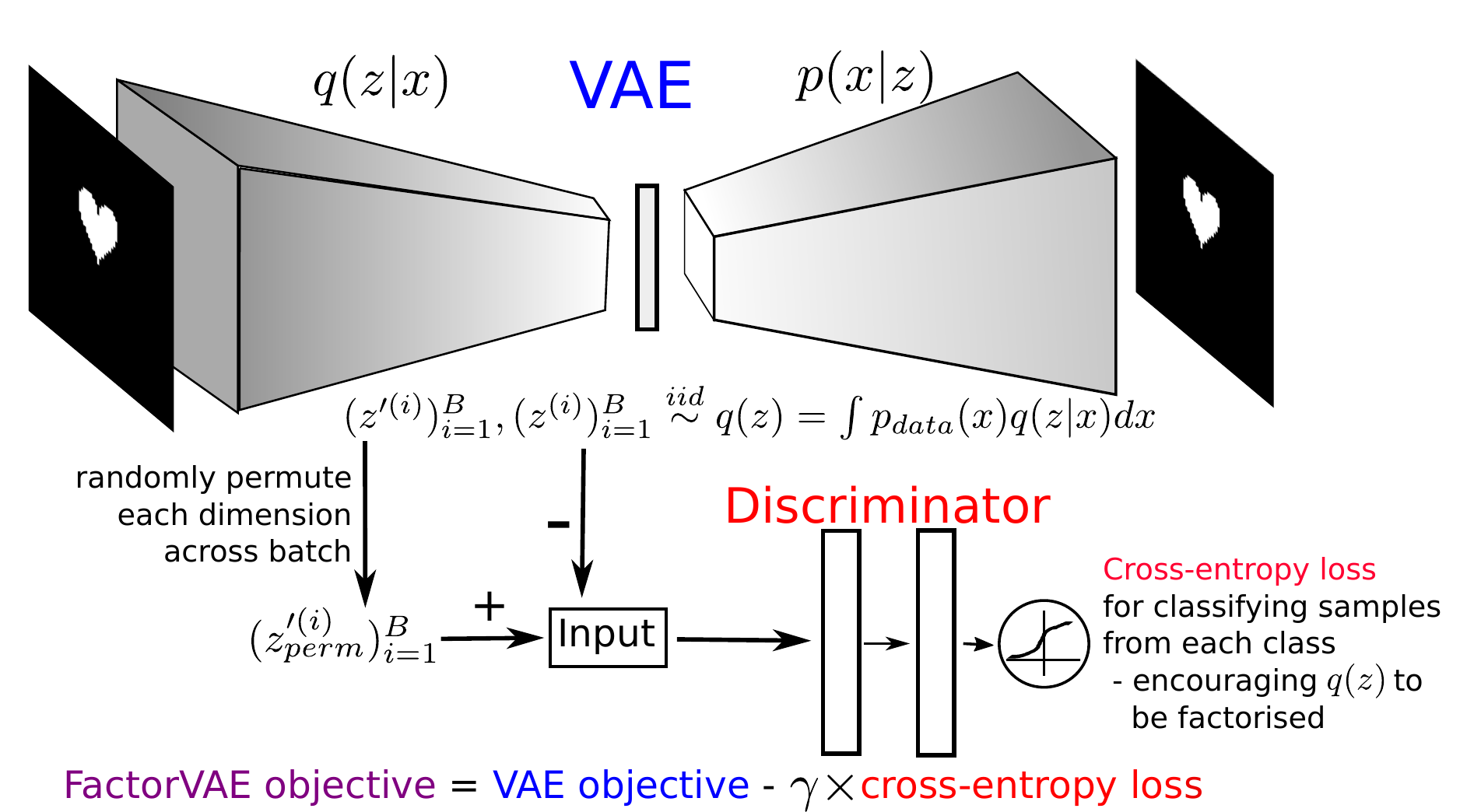}
  \vskip -0.1in
  \caption{Architecture of FactorVAE, a Variational Autoencoder (VAE) that encourages the code distribution to be factorial. The top row is a VAE with convolutional encoder and decoder, and the bottom row is an MLP classifier, the discriminator, that distinguishes whether the input was drawn from the marginal code distribution or the product of its marginals.}\label{fig:factor_vae}
\vskip -0.1in
\end{figure}

Using generative models has shown great promise in learning disentangled representations in images. Notably, semi-supervised approaches that require implicit or explicit knowledge about the true underlying factors of the data have excelled at disentangling \cite{kulkarni2015deep,kingma2014semi,reed2014learning,siddharth2017learning,hinton2011transforming,mathieu2016disentangling,goroshin2015unsupervised, hsu2017unsupervised, denton2017unsupervised}. However, ideally we would like to learn these in an unsupervised manner, due to the following reasons: 1. Humans are able to learn factors of variation unsupervised \cite{perry2010continuous}. 2. Labels are costly as obtaining them requires a human in the loop. 3. Labels assigned by humans might be inconsistent or leave out the factors that are difficult for humans to identify.

$\beta$-VAE \cite{higgins2016beta} is a popular method for unsupervised disentangling based on the 
Variational Autoencoder (VAE) framework \cite{kingma2013auto,rezende2014stochastic} for generative modelling. It uses a modified version of the VAE objective with a larger weight ($\beta > 1$) on the KL divergence between the variational posterior and the prior, and has proven to be an effective and stable method for disentangling.
One drawback of $\beta$-VAE is that reconstruction quality (compared to VAE) must be sacrificed in order to obtain better disentangling. The goal of our work is to obtain a better trade-off between disentanglement and reconstruction, allowing to achieve better disentanglement without degrading reconstruction quality. In this work, we analyse the source of this trade-off and propose FactorVAE, which augments the VAE objective with a penalty that encourages the marginal distribution of representations to be factorial without substantially affecting the quality of reconstructions. This penalty is expressed as a KL divergence between this marginal distribution and the product of its marginals, and is optimised using a discriminator network following the divergence minimisation view of GANs \cite{nowozin2016f,mohamed2016learning}. Our experimental results show that this approach achieves better disentanglement than $\beta$-VAE for the same reconstruction quality. We also point out the weaknesses in the disentangling metric of \citet{higgins2016beta}, and propose a new metric that addresses these shortcomings.

A popular alternative to $\beta$-VAE is InfoGAN \cite{chen2016infogan}, which is based on the Generative Adversarial Net (GAN) framework \cite{goodfellow2014generative} for generative modelling. InfoGAN learns disentangled representations by rewarding the mutual information between the observations and a subset of latents. However at least in part due to its training stability issues \cite{higgins2016beta}, there has been little empirical comparison between VAE-based methods and InfoGAN. Taking advantage of the recent developments in the GAN literature that help stabilise training, we include InfoWGAN-GP, a version of InfoGAN that uses Wasserstein distance \cite{arjovsky2017wasserstein} and gradient penalty \cite{gulrajani2017improved}, in our experimental evaluation.

In summary, we make the following contributions: 1) We introduce FactorVAE, a method for disentangling that gives higher disentanglement scores than $\beta$-VAE for the same reconstruction quality. 2) We identify the weaknesses of the disentanglement metric of \citet{higgins2016beta} and propose a more robust alternative. 3) We give quantitative comparisons of FactorVAE and $\beta$-VAE against InfoGAN's WGAN-GP counterpart for disentanglement.

\section{Trade-off between Disentanglement and Reconstruction in $\beta$-VAE}
We motivate our approach by analysing where the disentanglement and reconstruction trade-off arises in the $\beta$-VAE objective. First, we introduce notation and architecture of our VAE framework. We assume that observations $x^{(i)}\in \mathcal{X}, i=1,\ldots,N$ are generated by combining $K$ underlying factors $f=(f_1,\ldots,f_K)$. These observations are modelled using a real-valued latent/code vector $z \in \mathbb{R}^d$, interpreted as the representation of the data. The generative model is defined by the standard Gaussian prior $p(z)=\mathcal{N}(0,I)$, intentionally chosen to be a factorised distribution, and the decoder $p_{\theta}(x|z)$ parameterised by a neural net. The variational posterior for an observation is  $q_{\theta}(z|x) = \prod_{j=1}^d \mathcal{N}(z_j|\mu_j(x), \sigma_j^2(x))$, with the mean and variance produced by the encoder, also parameterised by a neural net.\footnote{In the rest of the paper we will omit the dependence of $p$ and $q$ on their parameters $\theta$ for notational convenience.} The variational posterior can be seen as the distribution of the representation corresponding to the data point $x$. The distribution of representations for the entire data set is then given by
\begin{equation}
q(z) = \mathbb{E}_{p_{data}(x)}[q(z|x)] = \frac{1}{N} \sum_{i=1}^N q(z|x^{(i)}),
\end{equation}
which is known as the marginal posterior or aggregate posterior, where $p_{data}$ is the empirical data distribution. 
A disentangled representation would have each $z_j$ correspond to precisely one underlying factor $f_k$. Since we assume that these factors vary independently, we wish for a factorial distribution  $q(z) = \prod_{j=1}^d q(z_j)$.

The $\beta$-VAE objective 
\begin{equation}
\frac{1}{N}\sum_{i=1}^N \left[ \mathbb{E}_{q(z|x^{(i)})}[\log p(x^{(i)}|z)] - \beta  KL(q(z|x^{(i)})||p(z)) \right] \nonumber
\end{equation}
is a variational lower bound on $\mathbb{E}_{p_{data}(x)}[ \log p(x^{(i)})]$ for $\beta \geq 1$, reducing to the VAE objective for $\beta = 1$. Its first term can be interpreted as the negative \textit{reconstruction error}, and the second term as the complexity penalty that acts as a regulariser. We may further break down this KL term as  \cite{hoffman2016elbo, makhzani2017pixelgan}
\begin{equation}
\mathbb{E}_{p_{data}(x)}[KL(q(z|x)||p(z))] = I(x;z) + KL(q(z)||p(z)), \nonumber
\end{equation}
where $I(x;z)$ is the mutual information between $x$ and $z$ under the joint distribution $p_{data}(x)q(z|x)$. See \apdx{apd:kl_decomp} for the derivation. Penalising the $KL(q(z)||p(z))$ term pushes $q(z)$ towards the factorial prior $p(z)$, encouraging independence in the dimensions of $z$ and thus disentangling. Penalising $I(x;z)$, on the other hand, reduces the amount of information about $x$ stored in $z$, which can lead to poor reconstructions for high values of $\beta$ \cite{makhzani2017pixelgan}. Thus making $\beta$ larger than 1, penalising both terms more, leads to better disentanglement but reduces reconstruction quality. When this reduction is severe, there is insufficient information about the observation in the latents, making it impossible to recover the true factors. Therefore there exists a value of $\beta > 1$ that gives highest disentanglement, but results in a higher reconstruction error than a VAE. 

\section{Total Correlation Penalty and FactorVAE}
\label{factorVAE}
Penalising $I(x;z)$ more than a VAE does might be neither necessary nor desirable for disentangling. For example, InfoGAN disentangles by encouraging $I(x;c)$ to be high where $c$ is a subset of the latent variables $z$ \footnote{Note however that $I(x;z)$ in $\beta$-VAE is defined under the joint distribution of data and their encoding distribution $p_{data}(x)q(z|x)$, whereas $I(x;c)$ in InfoGAN is defined on the joint distribution of the prior on $c$ and the decoding distribution $p(c)p(x|c)$.}. Hence we motivate FactorVAE by augmenting the VAE objective with a term that directly encourages independence in the code distribution, arriving at the following objective:
\begin{align} \label{eq:loss}
\frac{1}{N}\sum_{i=1}^N  \Big[ \mathbb{E}_{q(z|x^{(i)})}[\log p(x^{(i)}|z)] &- KL(q(z|x^{(i)})||p(z)) \Big] \nonumber \\ 
&- \gamma KL(q(z)||\bar{q}(z)),
\end{align}
where $\bar{q}(z) \coloneqq \prod_{j=1}^d q(z_j)$. Note that this is also a lower bound on the marginal log likelihood $\mathbb{E}_{p_{data}(x)}[\log p(x)]$. $KL(q(z)||\bar{q}(z))$ is known as \textit{Total Correlation} \citep[TC,][]{watanabe1960information}, a popular measure of dependence for multiple random variables. In our case this term is intractable since both $q(z)$ and $\bar{q}(z)$ involve mixtures with a large number of components, and the direct Monte Carlo estimate requires a pass through the entire data set for each $q(z)$ evaluation.\footnote{We have also tried using a batch estimate of $q(z)$, but this did not work. See \apdx{apd:batch_estimate_q} for details.}. Hence we take an alternative approach for optimizing this term. We start by observing we can sample from $q(z)$ efficiently by first choosing a datapoint $x^{(i)}$ uniformly at random and then sampling from $q(z|x^{(i)})$. We can also sample from $\bar{q}(z)$ by generating $d$ samples from $q(z)$ and then ignoring all but one dimension for each sample. A more efficient alternative involves sampling a batch from $q(z)$ and then randomly permuting across the batch for each latent dimension (see \alg{alg:permute_dims}). This is a standard trick used in the independence testing literature \cite{arcones1992bootstrap} and as long as the batch is large enough, the distribution of these samples samples will closely approximate $\bar{q}(z)$. 

Having access to samples from both distributions allows us to minimise their KL divergence using the \textit{density-ratio trick} \cite{nguyen2010estimating,sugiyama2012density} which involves training a classifier/discriminator to approximate the density ratio that arises in the KL term. Suppose we have a discriminator $D$ (in our case an MLP) that outputs an estimate of the probability $D(z)$ that its input is a sample from $q(z)$ rather than from $\bar{q}(z)$. Then we have
\begin{align} \label{eq:density_ratio}
TC(z) & = KL(q(z)||\bar{q}(z)) = \mathbb{E}_{q(z)}\bigg[ \log \frac{q(z)}{\bar{q}(z)}\bigg] \nonumber \\ 
& \approx \mathbb{E}_{q(z)}\bigg[ \log \frac{D(z)}{1-D(z)}\bigg].
\end{align}
We train the discriminator and the VAE jointly. In particular, the VAE parameters are updated using the objective in \eqn{eq:loss}, with the TC term replaced using the discriminator-based approximation from \eqn{eq:density_ratio}. The discriminator is trained to classify between samples from $q(z)$ and $\bar{q}(z)$, thus learning to approximate the density ratio needed for estimating TC. See \alg{alg:factor_vae} for pseudocode of FactorVAE.

\begin{algorithm}[h!]
   \caption{\texttt{permute\_dims}}
   \label{alg:permute_dims}
\begin{algorithmic}
   \STATE {\bfseries Input:} $\{z^{(i)} \in \mathbb{R}^d: i=1,\ldots,B\}$
   \FOR{$j=1$ {\bfseries to} $d$}
   \STATE $\pi \leftarrow$ random permutation on $\{1,\ldots,B\}$
   \STATE $(z^{(i)}_j)_{i=1}^B \leftarrow (z^{(\pi(i))}_j)_{i=1}^B$
   \ENDFOR
   \STATE {\bfseries Output:} $\{z^{(i)}: i=1,\ldots,B\}$
\end{algorithmic}
\end{algorithm}

\begin{algorithm}[h!]
   \caption{FactorVAE}
   \label{alg:factor_vae}
\begin{algorithmic}
   \STATE {\bfseries Input:} observations $(x^{(i)})_{i=1}^N$, batch size $m$, latent dimension $d$, $\gamma$, VAE/Discriminator optimisers: $g$, $g_D$
   \STATE Initialize VAE and discriminator parameters $\theta,\psi$.
   \REPEAT
   \STATE Randomly select batch $(x^{(i)})_{i \in \mathcal{B}}$ of size $m$
   \STATE Sample $z_{\theta}^{(i)} \sim q_{\theta}(z|x^{(i)})$ $\forall i \in \mathcal{B}$ 
   \STATE $\theta \leftarrow$ $g(\nabla_{\theta} \frac{1}{m} \sum \limits_{i \in \mathcal{B}} [\log \frac{p_{\theta}(x^{(i)},z_{\theta}^{(i)})}{q_{\theta}(z_{\theta}^{(i)}|x^{(i)})} - \gamma \log \frac{D_{\psi}(z_{\theta}^{(i)})}{1-D_{\psi}(z_{\theta}^{(i)})}])$
   \STATE Randomly select batch $(x^{(i)})_{i \in \mathcal{B'}}$ of size $m$
   \STATE Sample $z'^{(i)}_{\theta} \sim q_{\theta}(z|x^{(i)})$ for $i \in \mathcal{B'}$
   \STATE $(z'^{(i)}_{perm})_{i \in \mathcal{B'}} \leftarrow$ \texttt{permute\_dims}($(z'^{(i)}_{\theta})_{i \in \mathcal{B'}}$)
   \STATE $\psi \leftarrow g_D(\nabla_{\psi} \frac{1}{2m} [\sum \limits_{i \in \mathcal{B}} \log(D_{\psi}(z^{(i)}_{\theta}))$
   \STATE \hspace{25mm} $+\sum \limits_{i \in \mathcal{B'}}\log(1-D_{\psi}(z'^{(i)}_{perm}))])$
   \UNTIL convergence of objective.
\end{algorithmic}
\end{algorithm}

It is important to note that low TC is necessary but not sufficient for meaningful disentangling. For example, when $q(z|x)=p(z)$, TC=0 but $z$ carries no information about the data. Thus having low TC is only meaningful when we can preserve information in the latents, which is why controlling for reconstruction error is important.

In the GAN literature, divergence minimisation is usually done between two distributions over the data space, which is often very high dimensional (\eg images). As a result, the two distributions often have disjoint support, making training unstable, especially when the discriminator is strong. Hence it is necessary to use tricks to weaken the discriminator such as instance noise \cite{sonderby2016amortised} or to replace the discriminator with a critic, as in Wasserstein GANs \cite{arjovsky2017wasserstein}. In this work, we minimise divergence between two distributions over the latent space (as in \eg  \cite{Mescheder2017ICML}), which is typically much lower dimensional and the two distributions have overlapping support. We observe that training is stable for sufficiently large batch sizes (\eg 64 worked well for $d=10$), allowing us to use a strong discriminator.

\section{A New Metric for Disentanglement}
\label{metric}

\begin{figure}[h!]
\vskip -0.1in
  \centering
  \includegraphics[width=\columnwidth]{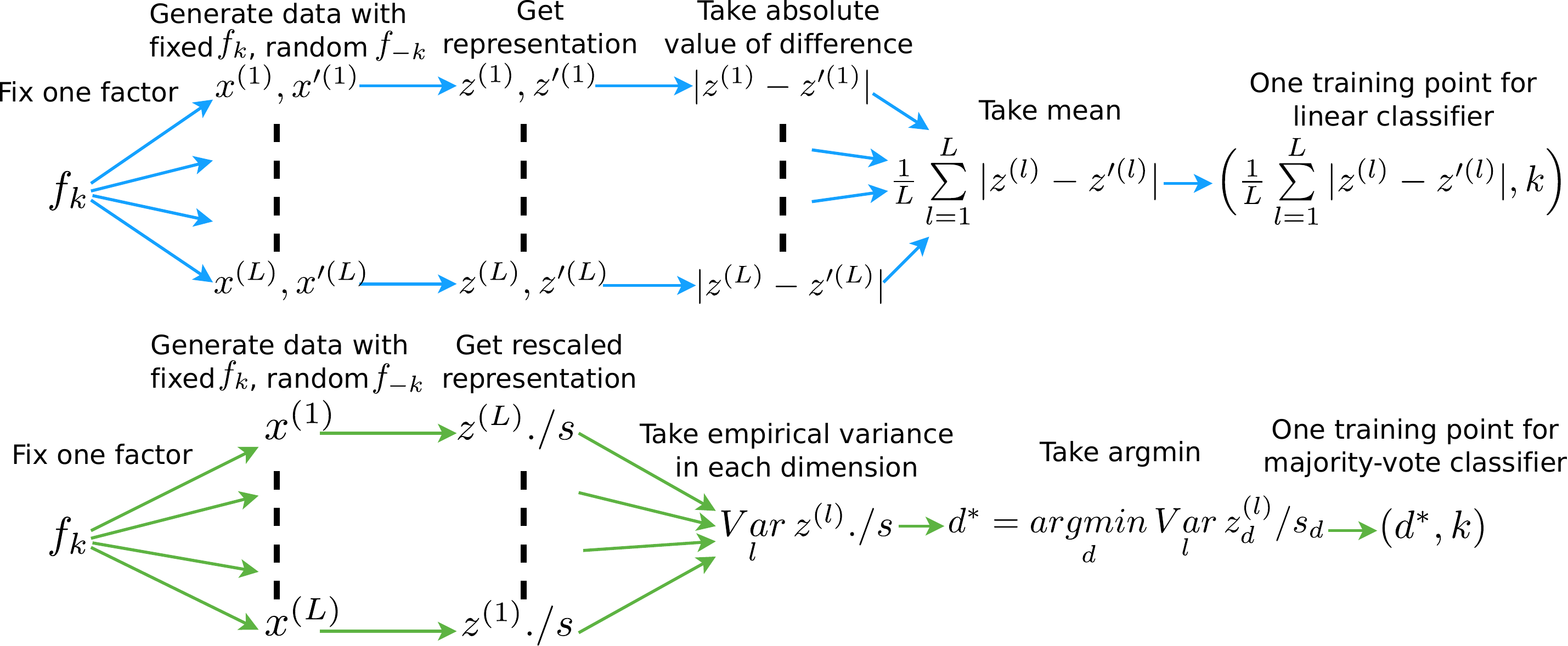}
    \vskip -0.1in
  \caption{Top: Metric in \cite{higgins2016beta}. Bottom: Our new metric, where $s \in \mathbb{R}^d$ is the scale (empirical standard deviation) of latent representations of the full data (or large enough random subset).}\label{fig:metric}
\vskip -0.1in
\end{figure}

The definition of disentanglement we use in this paper, where a change in one dimension of the representation corresponds to a change in exactly one factor of variation, is clearly a simplistic one. It does not allow correlations among the factors or hierarchies over them. Thus this definition seems more suited to synthetic data with independent factors of variation than to most realistic data sets. However, as we will show below, robust disentanglement is not a fully solved problem even in this simple setting. One obstacle on the way to this first milestone is the absence of a sound quantitative metric for measuring disentanglement.

A popular method of measuring disentanglement is by inspecting \textit{latent traversals}: visualising the change in reconstructions while traversing one dimension of the latent space at a time. Although latent traversals can be a useful indicator of when a model has failed to disentangle, the qualitative nature of this approach makes it unsuitable for comparing algorithms reliably. Doing this would require inspecting a multitude of latent traversals over multiple reference images, random seeds, and points during training.
Having a human in the loop to assess the traversals is also too time-consuming and subjective. Unfortunately, for data sets that do not have the ground truth factors of variation available, currently this is the only viable option for assessing disentanglement.

\citet{higgins2016beta} proposed a supervised metric that attempts to quantify disentanglement  when the ground truth factors of a data set are given. The metric is the error rate of a linear classifier that is trained as follows. Choose a factor $k$; generate data with this factor fixed but all other factors varying randomly; obtain their representations (defined to be the mean of $q(z|x)$); take the absolute value of the pairwise differences of these representations. Then the mean of these statistics across the pairs gives one training input for the classifier, and the fixed factor index $k$ is the corresponding training output (see top of \fig{fig:metric}). So if the representations were perfectly disentangled, we would see zeros in the dimension of the training input that corresponds to the fixed factor of variation, and the  classifier would learn to map the index of the zero value to the index of the factor. 

However this metric has several weaknesses. Firstly, it could be sensitive to hyperparameters of the linear classifier optimisation, such as the choice of the optimiser and its hyperparameters, weight initialisation, and the number of training iterations. 
Secondly, having a linear classifier is not so intuitive -- we could get representations where each factor corresponds to a linear combination of dimensions instead of a single dimension. Finally and most importantly, the metric has a failure mode: it gives 100\% accuracy even when only $K-1$ factors out of $K$ have been disentangled; to predict the remaining factor, the classifier simply learns to detect when all the values corresponding to the $K-1$ factors are non-zero. An example of such a case is shown in \fig{fig:failure}.

\begin{figure}[t!] 
\vskip -0.1in
  \centering
  \includegraphics[width=0.5\linewidth]{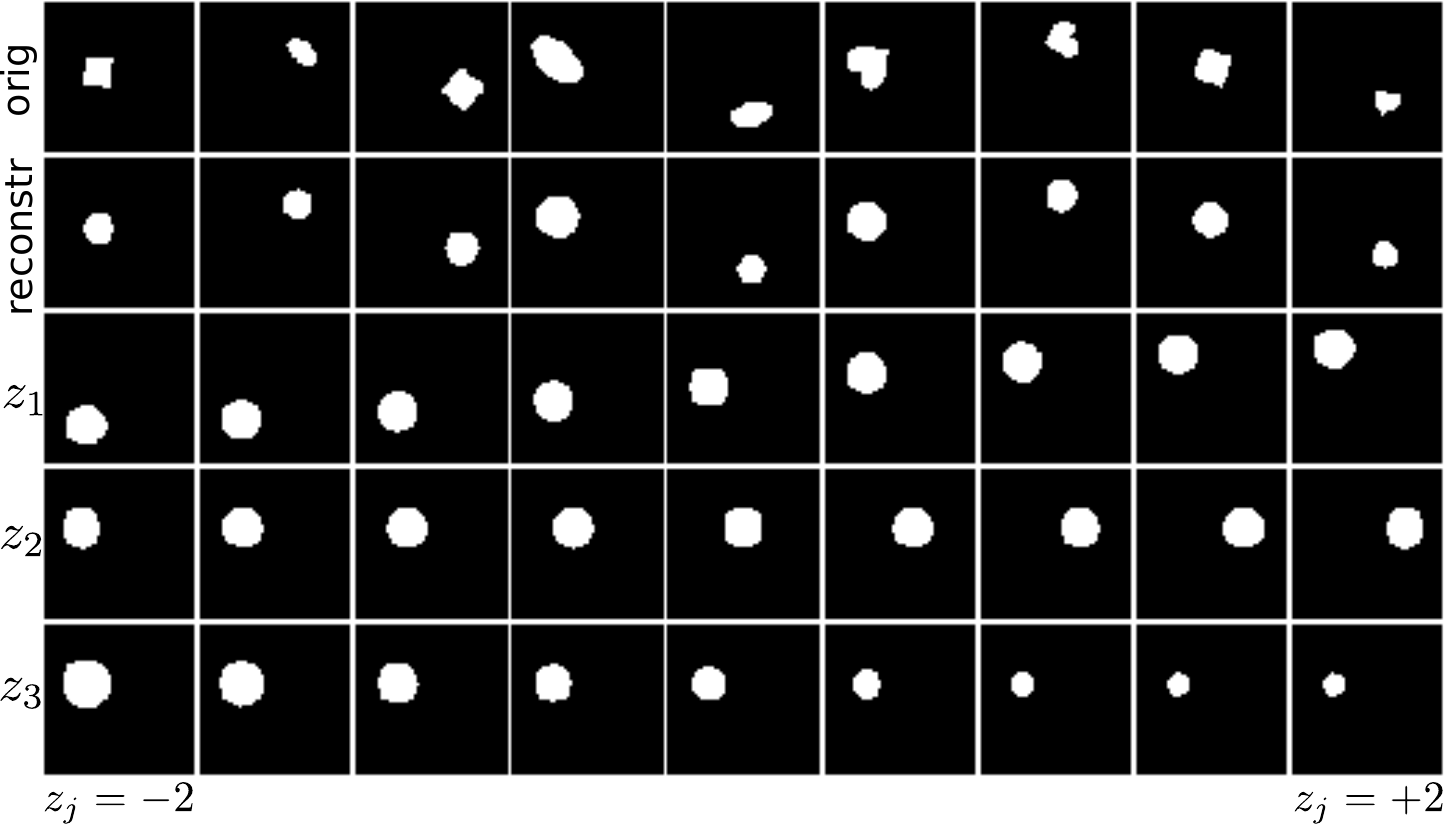}
  \vskip -0.1in
  \caption{A $\beta$-VAE model trained on the 2D Shapes data that scores 100\% on metric in \citet{higgins2016beta} (ignoring the shape factor). First row: originals. Second row: reconstructions. Remaining rows: reconstructions of latent traversals. The model only uses three latent units to capture $x$-position, $y$-position, scale and ignores orientation, yet achieves a perfect score on the metric.}\label{fig:failure}
\vskip -0.2in
\end{figure}

To address these weaknesses, we propose a new disentanglement metric as follows. Choose a factor $k$; generate data with this factor fixed but all other factors varying randomly; obtain their representations; normalise each dimension by its empirical standard deviation over the full data (or a large enough random subset); take the empirical variance in each dimension\footnote{We can use Gini's definition of variance for discrete latents \cite{gini1971variability}. See \apdx{apd:metric} for details.} of these normalised representations. Then the index of the dimension with the lowest variance and the target index $k$ provide one training input/output example for the classifier (see bottom of \fig{fig:metric}). Thus if the representation is perfectly disentangled, the empirical variance in the dimension corresponding to the fixed factor will be 0. We normalise the representations so that the $\argmin$ is invariant to rescaling of the representations in each dimension. Since both inputs and outputs lie in a discrete space, the optimal classifier is the majority-vote classifier (see \apdx{apd:metric} for details), and the metric is the error rate of the classifier. The resulting classifier is a deterministic function of the training data, hence there are no optimisation hyperparameters to tune. We also believe that this metric is conceptually simpler and more natural than the previous one. Most importantly, it circumvents the failure mode of the earlier metric, since the classifier needs to see the lowest variance in a latent dimension for a given factor to classify it correctly.

We think developing a reliable unsupervised disentangling metric that does not use the ground truth factors is an important direction for future research, since unsupervised disentangling is precisely useful for the scenario where we do not have access to the ground truth factors. With this in mind, we believe that having a reliable supervised metric is still valuable as it can serve as a gold standard for evaluating unsupervised metrics.

\section{Related Work}
\label{related}
There are several recent works that use a discriminator to optimise a divergence to encourage independence in the latent codes. Adversarial Autoencoder \citep[AAE,][]{makhzani2015adversarial} removes the $I(x;z)$ term in the VAE objective and maximizes the negative reconstruction error minus $KL(q(z)||p(z))$ via the density-ratio trick, showing applications in semi-supervised classification and unsupervised clustering. This means that the AAE objective is not a lower bound on the log marginal likelihood. Although optimising a lower bound is not strictly necessary for disentangling, it does ensure that we have a valid generative model; having a generative model with disentangled latents has the benefit of being a single model that can be useful for various tasks e.g. planning for model-based RL, visual concept learning and semi-supervised learning, to name a few. In PixelGAN Autoencoders \cite{makhzani2017pixelgan}, the same objective is used to study the decomposition of information between the latent code and the decoder.  The authors state that adding noise to the inputs of the encoder is crucial, which suggests that limiting the information that the code contains about the input is essential and that the $I(x;z)$ term should not be dropped from the VAE objective. \citet{brakel2017learning} also use a discriminator to penalise the Jensen-Shannon Divergence between the distribution of codes and the product of its marginals. However, they use the GAN loss with deterministic encoders and decoders and only explore their technique in the context of Independent Component Analysis source separation. 

Early works on unsupervised disentangling include \citep{schmidhuber1992learning} which attempts to disentangle codes in an autoencoder by penalising predictability of one latent dimension given the others and \citep{desjardins2012disentangling} where a variant of a Boltzmann Machine is used to disentangle two factors of variation in the data. More recently, \citet{achille2018information} have used a loss function that penalises TC in the context of supervised learning. They show that their approach can be extended to the VAE setting, but do not perform any experiments on disentangling to support the theory. In a concurrent work, \citet{kumar2017variational} used moment matching in VAEs to penalise the covariance between the latent dimensions, but did not constrain the mean or higher moments. We provide the objectives used in these related methods and show experimental results on disentangling performance, including AAE, in \apdx{apd:related}.

There have been various works that use the notion of predictability to quantify disentanglement, mostly predicting the value of ground truth factors $f=(f_1,\ldots,f_K)$ from the latent code $z$. This dates back to \citet{yang1997adaptive} who learn a linear map from representations to factors in the context of linear ICA, and quantify how close this map is to a permutation matrix. More recently \citet{eastwood2018framework} have extended this idea to disentanglement by training a Lasso regressor to map $z$ to $f$ and using its trained weights to quantify disentanglement. Like other regression-based approaches, this one introduces hyperparameters such as the optimiser and the Lasso penalty coefficient. The metric of \citet{higgins2016beta} as well as the one we proposed, predict the factor $k$ from the $z$ of images with a fixed $f_k$ but $f_{-k}$ varying randomly. \citet{schmidhuber1992learning} quantifies predictability between the different dimensions of $z$, using a predictor that is trained to predict $z_j$ from $z_{-j}$. 

\textit{Invariance} and \textit{equivariance} are frequently considered to be desirable properties of representations in the literature \cite{goodfellow2009measuring,kivinen2011transformation,lenc2015understanding}. A representation is said to be \textit{invariant} for a particular task if it does not change when nuisance factors of the data, that are irrelevant to the task, are changed. An \textit{equivariant} representation changes in a stable and predictable manner when altering a factor of variation. A disentangled representation, in the sense used in the paper, is equivariant, since changing one factor of variation will change one dimension of a disentangled representation in a predictable manner. Given a task, it will be easy to obtain an invariant representation from the disentangled representation by ignoring the dimensions encoding the nuisance factors for the task \cite{cohen2014learning}.

Building on a preliminary version of this paper, \cite{chen2018isolating} recently proposed a minibatch-based alternative to our density-ratio-trick-based method for estimating the Total Correlation and introduced an information-theoretic disentangling metric.

\section{Experiments} \label{exp_vae}

We compare FactorVAE to $\beta$-VAE on the following data sets with i) known generative factors: 1) \textbf{2D Shapes} \cite{dsprites17}: 737,280 binary $64\times64$ images of 2D shapes with ground truth factors[number of values]: shape[3], scale[6], orientation[40], x-position[32], y-position[32]. 2) \textbf{3D Shapes} \cite{3dshapes18}: 480,000 RGB $64\times64\times3$ images of 3D shapes with ground truth factors: shape[4], scale[8], orientation[15], floor colour[10], wall colour[10], object colour[10] ii) unknown generative factors: 3) \textbf{3D Faces} \cite{paysan20093d}: 239,840 grey-scale $64\times64$ images of 3D Faces. 4) \textbf{3D Chairs} \cite{aubry2014seeing}: 86,366 RGB $64\times64\times3$ images of chair CAD models. 5) \textbf{CelebA} (cropped version) \cite{liu2015deep}: 202,599 RGB $64\times64\times3$ images of celebrity faces. The experimental details such as encoder/decoder architectures and hyperparameter settings are in \apdx{apd:exp}. The details of the disentanglement metrics, along with a sensitivity analysis with respect to their hyperparameters, are given in \apdx{apd:metric}.

\begin{figure}[h!]
\vskip -0.1in
  \centering
  \includegraphics[width=\columnwidth]{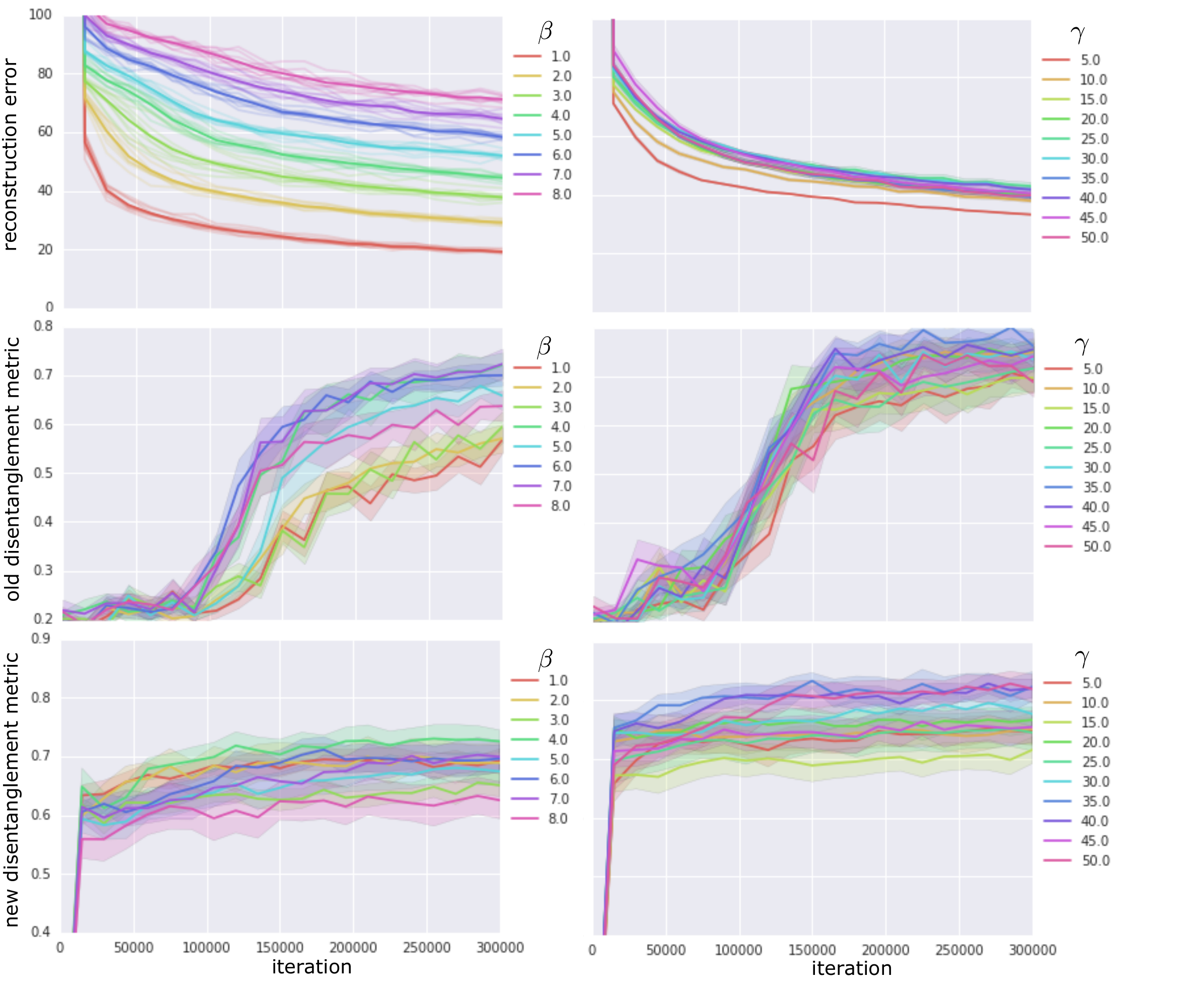}
  \vskip -0.1in
  \caption{Reconstruction error (top), metric in \citet{higgins2016beta} (middle), our metric (bottom). $\beta$-VAE (left), FactorVAE (right). The colours correspond to different values of $\beta$ and $\gamma$ respectively, and confidence intervals are over 10 random seeds.}\label{fig:shapes2d_plots}
\vskip -0.1in
\end{figure}

\begin{figure}[h!]
  \centering
  \includegraphics[width=0.65\columnwidth]{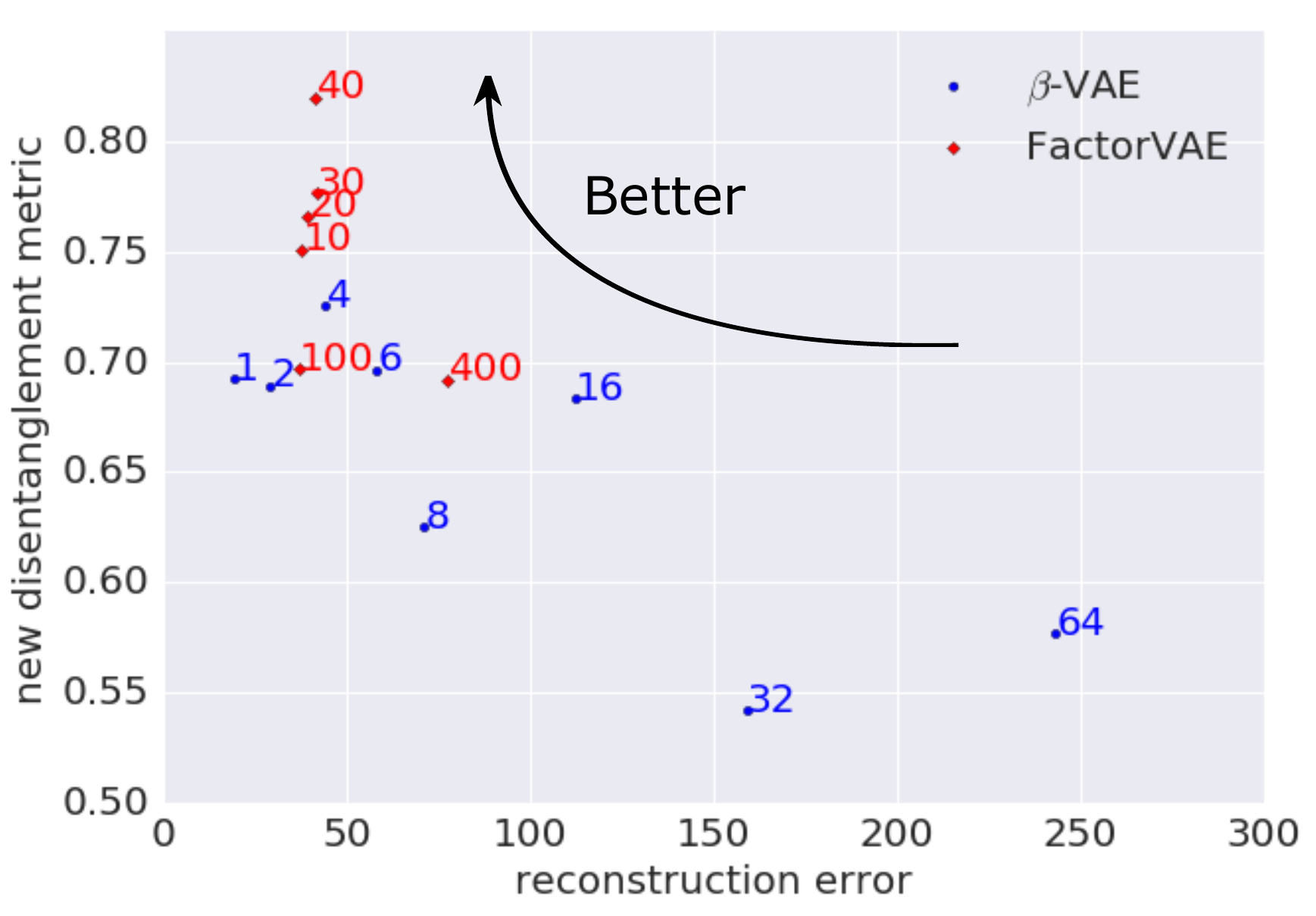}
  \vskip -0.1in
  \caption{Reconstruction error plotted against our disentanglement metric, both averaged over 10 random seeds at the end of training. The numbers at each point are values of $\beta$ and $\gamma$. Note that we want low reconstruction error and a high disentanglement metric.}\label{fig:shapes2d_tradeoff}
\end{figure}

\begin{figure}[h!]
\vskip -0.1in
  \centering
  \includegraphics[width=\columnwidth]{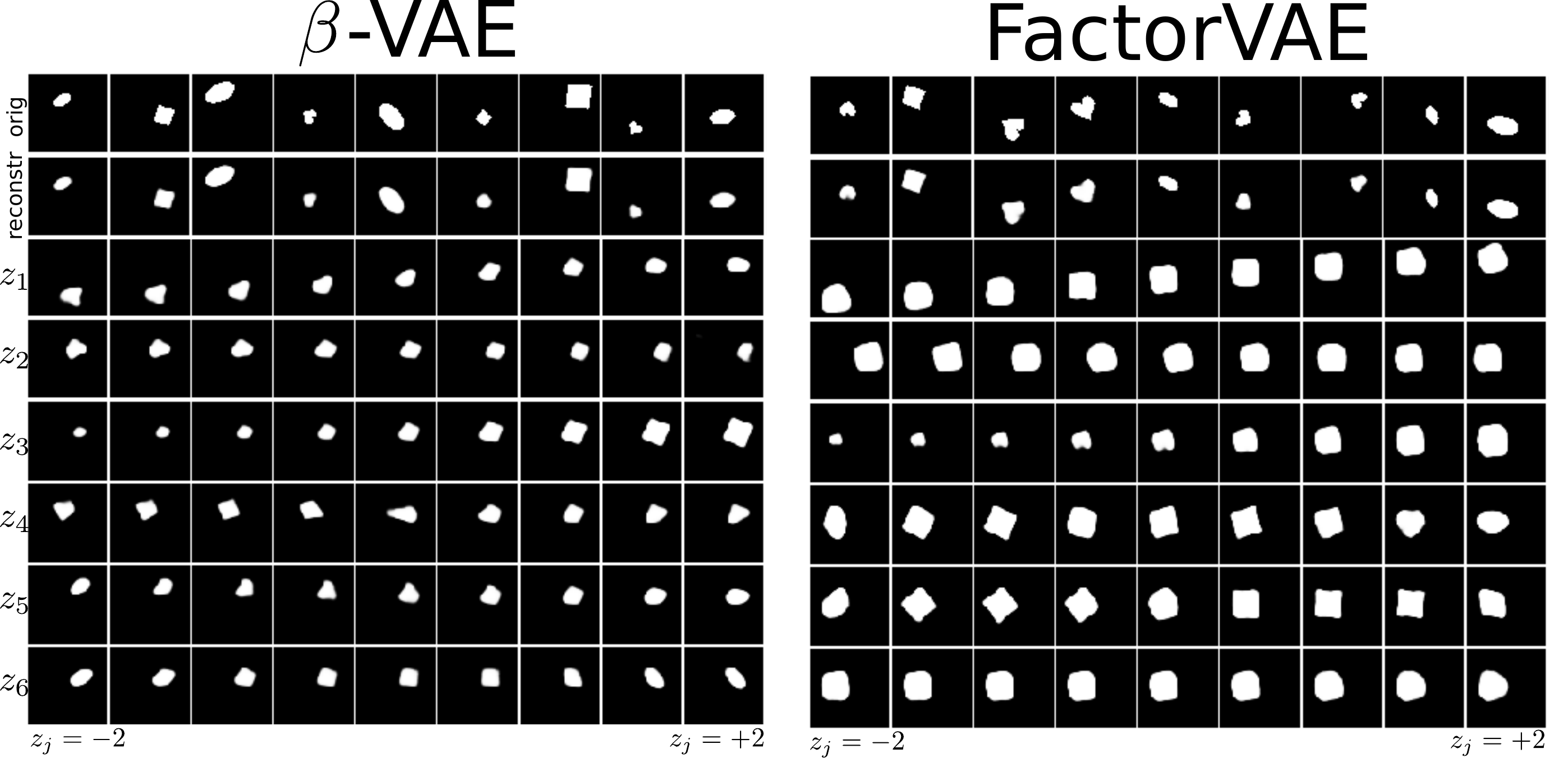}
  \vskip -0.1in
  \caption{First row: originals. Second row: reconstructions. Remaining rows: reconstructions of latent traversals across each latent dimension sorted by $KL(q(z_j|x)||p(z_j))$, for the best scoring models on our disentanglement metric. Left: $\beta$-VAE, score: 0.814, $\beta=4$. Right: FactorVAE, score: 0.889, $\gamma=35$.}\label{fig:shapes2d_latent_traversals}
\vskip -0.1in
\end{figure}

From \fig{fig:shapes2d_plots}, we see that FactorVAE gives much better disentanglement scores than VAEs ($\beta=1$), while barely sacrificing reconstruction error, highlighting the disentangling effect of adding the Total Correlation penalty to the VAE objective. The best disentanglement scores for FactorVAE are noticeably better than those for $\beta$-VAE given the same reconstruction error. This can be seen more clearly in Figure \ref{fig:shapes2d_tradeoff} where the best mean disentanglement of FactorVAE ($\gamma=40$) is around 0.82, significantly higher than the one for $\beta$-VAE ($\beta=4$), which is around 0.73, both with reconstruction error around 45. From Figure \ref{fig:shapes2d_latent_traversals}, we can see that both models are capable of finding $x$-position, $y$-position, and scale, but struggle to disentangle orientation and shape, $\beta$-VAE especially. For this data set, neither method can robustly capture shape, the discrete factor of variation\footnote{This is partly due to the fact that learning discrete factors would require using discrete latent variables instead of Gaussians, but jointly modelling discrete and continuous factors of variation is a non-trivial problem that needs further research.}.

As a sanity check, we also evaluated the correlation between our metric and the metric in \citet{higgins2016beta}: \textbf{Pearson} (linear correlation coefficient): 0.404, \textbf{Kendall} (proportion of pairs that have the same ordering): 0.310, \textbf{Spearman} (linear correlation of the rankings): 0.444, all with p-value 0.000. Hence the two metrics show a fairly high positive correlation as expected. 

\begin{figure}[t!]
\vskip -0.1in
  \centering
  \includegraphics[width=0.9\columnwidth]{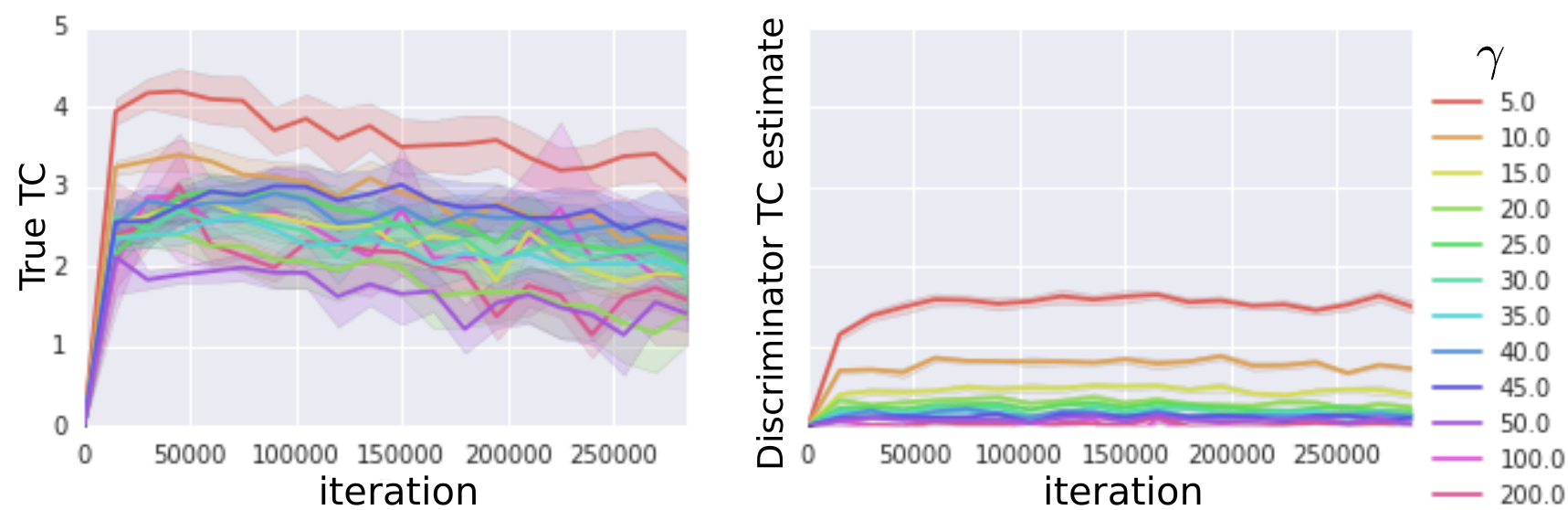}
  \vskip -0.1in
  \caption{Total Correlation values for FactorVAE on 2D Shapes. Left: True TC value. Right: Discriminator's estimate of TC.}\label{fig:shapes2d_tc}
\vskip -0.1in
\end{figure}

We have also examined how the discriminator's estimate of the Total Correlation (TC) behaves and the effect of  $\gamma$ on the true TC. From \fig{fig:shapes2d_tc}, observe that the discriminator is consistently underestimating the true TC, also confirmed in \cite{rosca2018distribution}. However the true TC decreases throughout training, and a higher $\gamma$ leads to lower TC, so the gradients obtained using the discriminator are sufficient for encouraging independence in the code distribution.

We then evaluated InfoWGAN-GP, the counterpart of InfoGAN that uses Wasserstein distance and gradient penalty. See \apdx{apd:infogan_overview} for an overview. One advantage of InfoGAN is that the Monte Carlo estimate of its objective is differentiable with respect to its parameters even for discrete codes $c$, which makes gradient-based optimisation straightforward. In contrast, VAE-based methods that rely on the reprameterisation trick for gradient-based optimisation require $z$ to be a reparameterisable continuous random variable and alternative approaches require various variance reduction techniques for gradient estimation   \cite{MnihRezende2016,maddison2016concrete}. Thus we might expect Info(W)GAN(-GP) to show better disentangling in cases where some factors are discrete. Hence we use 4 continuous latents (one for each continuous factor) and one categorical latent of 3 categories (one for each shape). We tuned for $\lambda$, the weight of the mutual information term in Info(W)GAN(-GP), $\in \{0.0,0.1,0.2,\ldots,1.0\}$, number of noise variables $\in \{5,10,20,40,80,160\}$ and the learning rates of the generator $\in \{10^{-3}, 10^{-4}\}$, discriminator $\in \{10^{-4}, 10^{-5}\}$.

\begin{figure}[h!]
  \centering
  \includegraphics[width=\columnwidth]{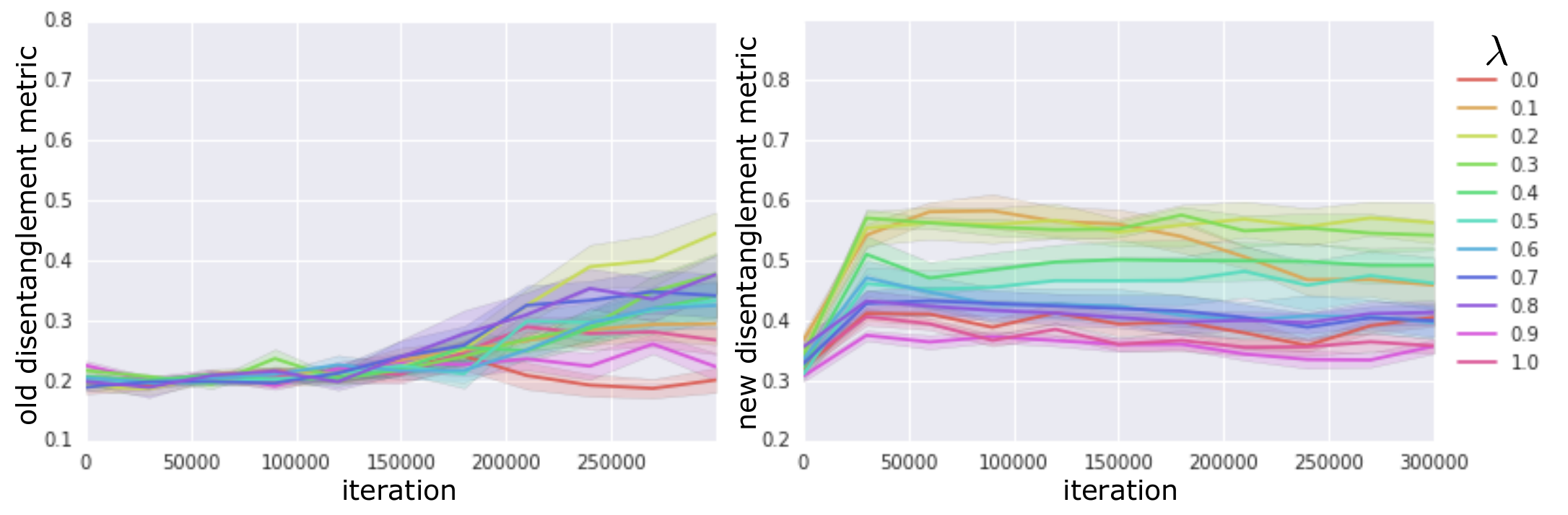}
  \vskip -0.1in
  \caption{Disentanglement scores for InfoWGAN-GP on 2D Shapes for 10 random seeds per hyperparameter setting. Left: Metric in \citet{higgins2016beta}. Right: Our metric.}\label{fig:infowgan_gp_shapes2d_disent}
\vskip -0.1in
\end{figure}

\begin{figure}[h!]
  \centering
  \includegraphics[width=0.7\columnwidth]{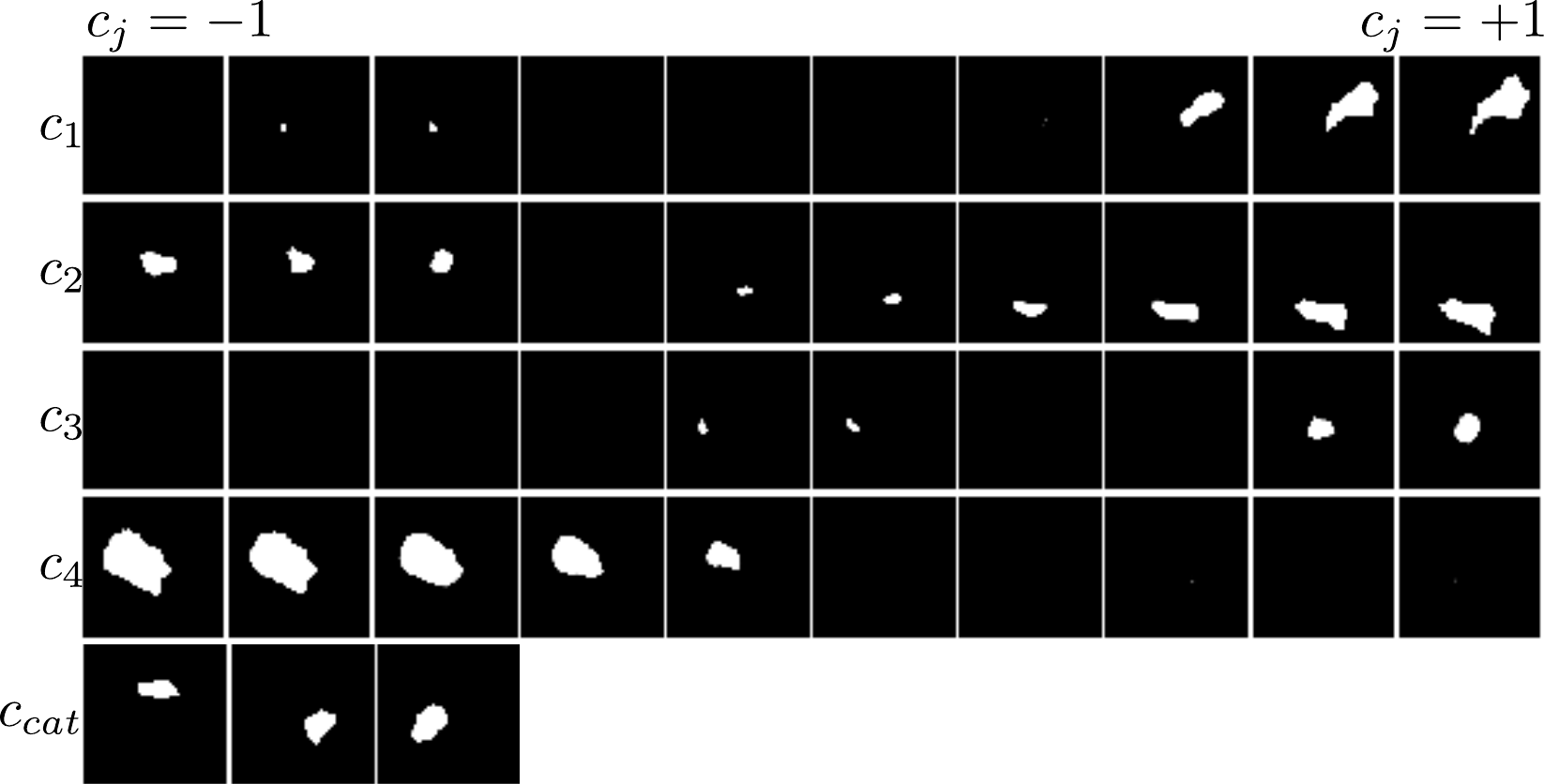}
  \vskip -0.1in
  \caption{Latent traversals for InfoWGAN-GP on 2D Shapes across four continuous codes (first four rows) and categorical code (last row) for run with best disentanglement score ($\lambda=0.2$).}\label{fig:infowgan_gp_shapes2d_latent_traversals}
\end{figure}

However from \fig{fig:infowgan_gp_shapes2d_disent} we can see that the disentanglement scores are disappointingly low. From the latent traversals in  \fig{fig:infowgan_gp_shapes2d_latent_traversals}, we can see that the model learns only the scale factor, and tries to put positional information in the discrete latent code, which is one reason for the low disentanglement score. Using 5 continuous codes and no categorical codes did not improve the disentanglement scores however. InfoGAN with early stopping (before training instability occurs -- see \apdx{apd:infogan}) also gave similar results. The fact that some  latent traversals give blank reconstructions indicates that the model does not generalise well to all parts of the domain of $p(z)$.

One reason InfoWGAN-GP's poor performance on this data set could be that InfoGAN is sensitive to the generator and discriminator architecture, which is one thing we did not tune extensively. We use a similar architecture to the VAE-based approaches for 2D shapes for a fair comparison, but have also tried a bigger architecture which gave similar results (see \apdx{apd:infogan}). If architecture search is indeed important, this would be a weakness of InfoGAN relative to FactorVAE and $\beta$-VAE, which are both much more robust to architecture choice. In \apdx{apd:infogan}, we check that we can replicate the results of \citet{chen2016infogan} on MNIST using InfoWGAN-GP, verify that it makes training stable compared to InfoGAN, and give implementation details with further empirical studies of InfoGAN and InfoWGAN-GP. 

We now show results on the 3D Shapes data, which is a more complex data set of 3D scenes with additional features such as shadows and background (sky). We train both $\beta$-VAE and FactorVAE for 1M iterations. \fig{fig:shapes3d_tradeoff} again shows that FactorVAE achieves much better disentanglement with barely any increase in reconstruction error compared to VAE. Moreover, while the top mean disentanglement scores for FactorVAE and $\beta$-VAE are similar, the reconstruction error is lower for FactorVAE: 3515 ($\gamma=36$) as compared to 3570 ($\beta=24$). The latent traversals in \fig{fig:shapes3d_latent_traversals} show that both models are able to capture the factors of variation in the best-case scenario. Looking at latent traversals across many random seeds, however, makes it evident that both models struggled to disentangle the factors for shape and scale.

\begin{figure}[t!]
\vskip -0.1in
  \centering
  \includegraphics[width=0.65\columnwidth]{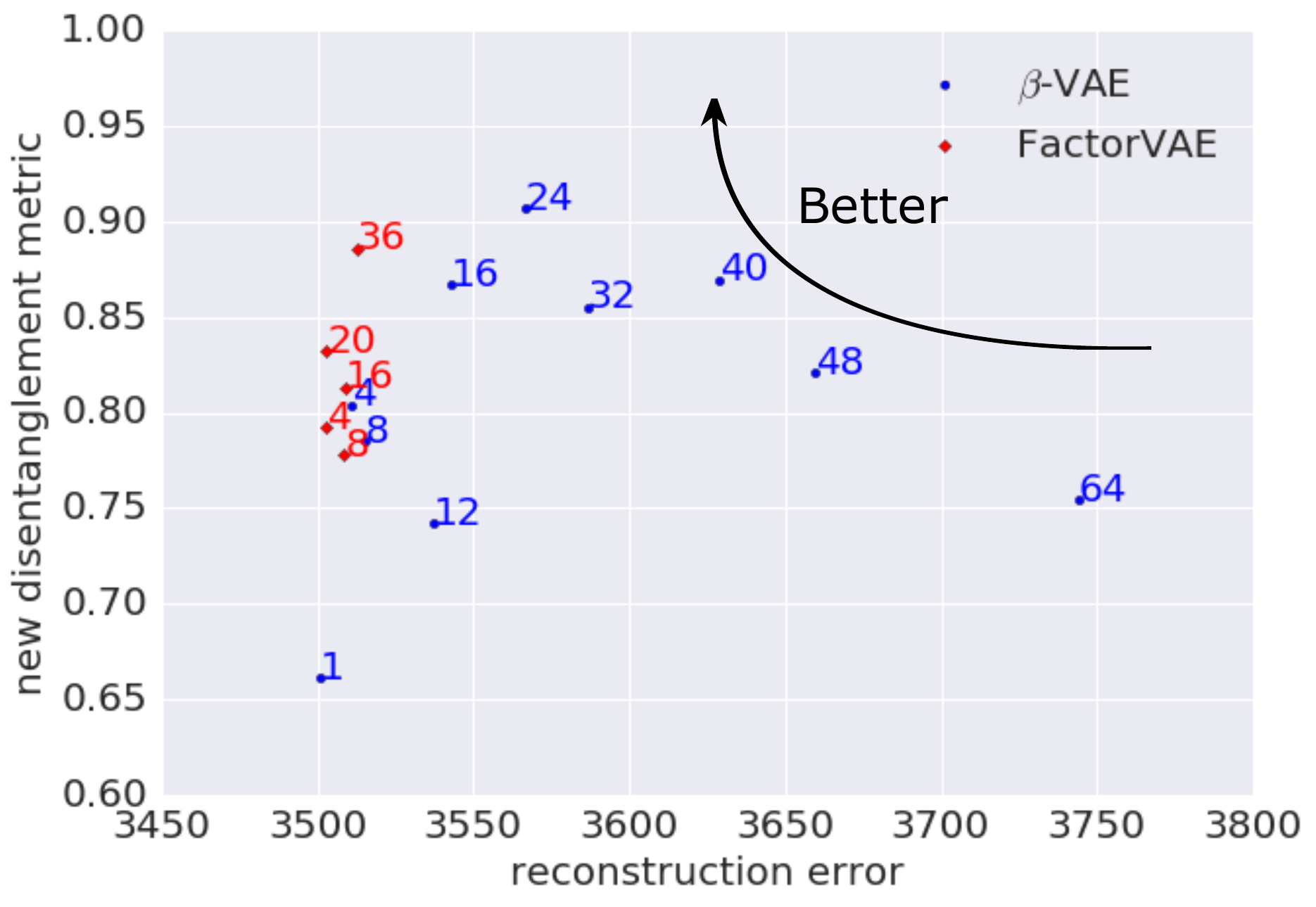}
  \vskip -0.1in
  \caption{Same as \fig{fig:shapes2d_tradeoff} for 3D Shapes data.}\label{fig:shapes3d_tradeoff}
\vskip -0.1in
\end{figure}

\begin{figure}[t!]
  \centering
  \includegraphics[width=\columnwidth]{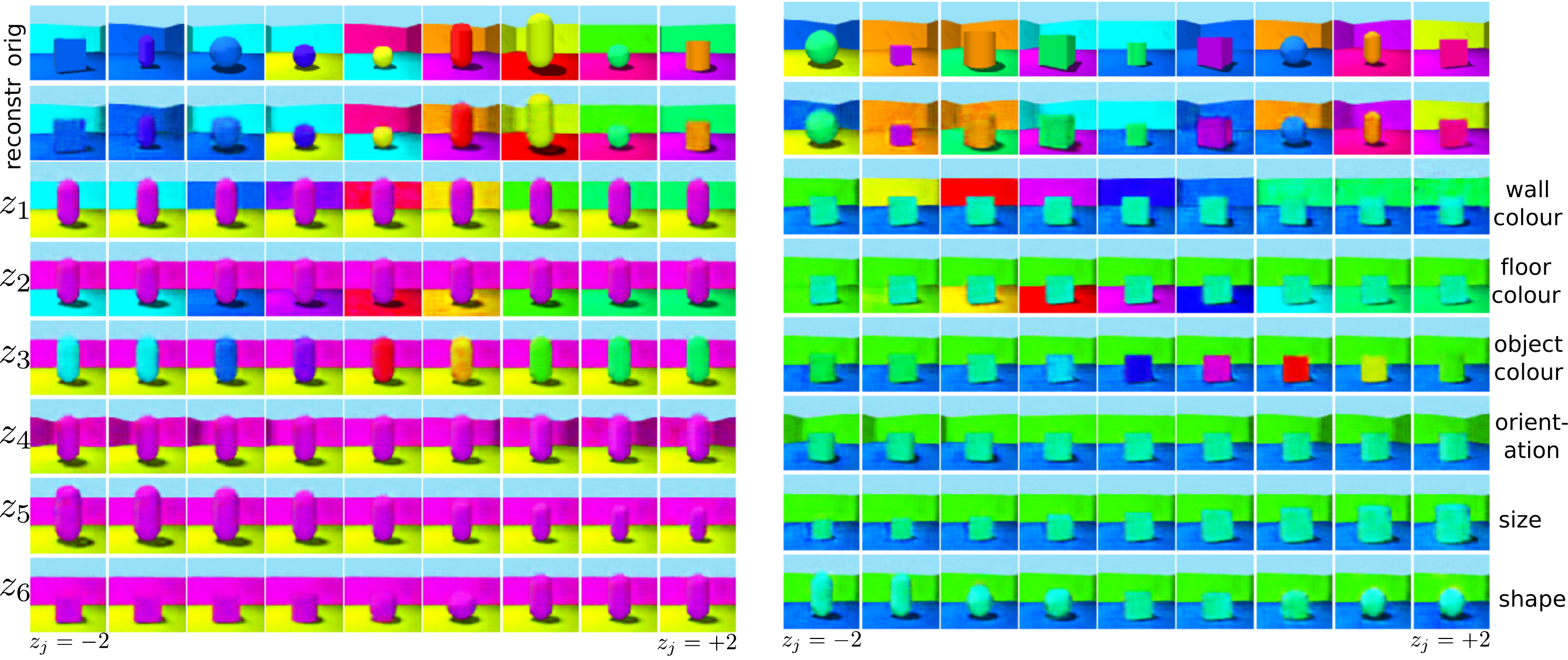}
  \vskip -0.1in
  \caption{Same as \fig{fig:shapes2d_latent_traversals} but for 3D Shapes data. Left: $\beta$-VAE, score: 1.00, $\beta=32$. Right: FactorVAE, score: 1.00, $\gamma=7$.}\label{fig:shapes3d_latent_traversals}
\vskip -0.1in
\end{figure}

To show that FactorVAE also gives a valid generative model for both 2D Shapes and 3D Shapes, we present the log marginal likelihood evaluated on the entire data set together with samples from the generative model in \apdx{apd:generative}.



We also show results for $\beta$-VAE and FactorVAE experiments on the data sets with unknown generative factors, namely 3D Chairs, 3D Faces, and CelebA. Note that inspecting latent traversals is the only evaluation method possible here. We can see from Figure \ref{fig:faces3d_rec_err} (and Figures~\ref{fig:chairs_rec_err} and \ref{fig:celeba_rec_err} in \apdx{apd:further_results}) that FactorVAE has smaller reconstruction error compared to $\beta$-VAE, and is capable of learning sensible factors of variation, as shown in the latent traversals in Figures~\ref{fig:chairs_latent_traversals}, \ref{fig:faces3d_latent_traversals} and \ref{fig:celeba_latent_traversals}. Unfortunately, as explained in \sect{metric}, latent traversals tell us little about the robustness of our method.

\begin{figure}[h!]
  \centering
  \includegraphics[width=0.8\columnwidth]{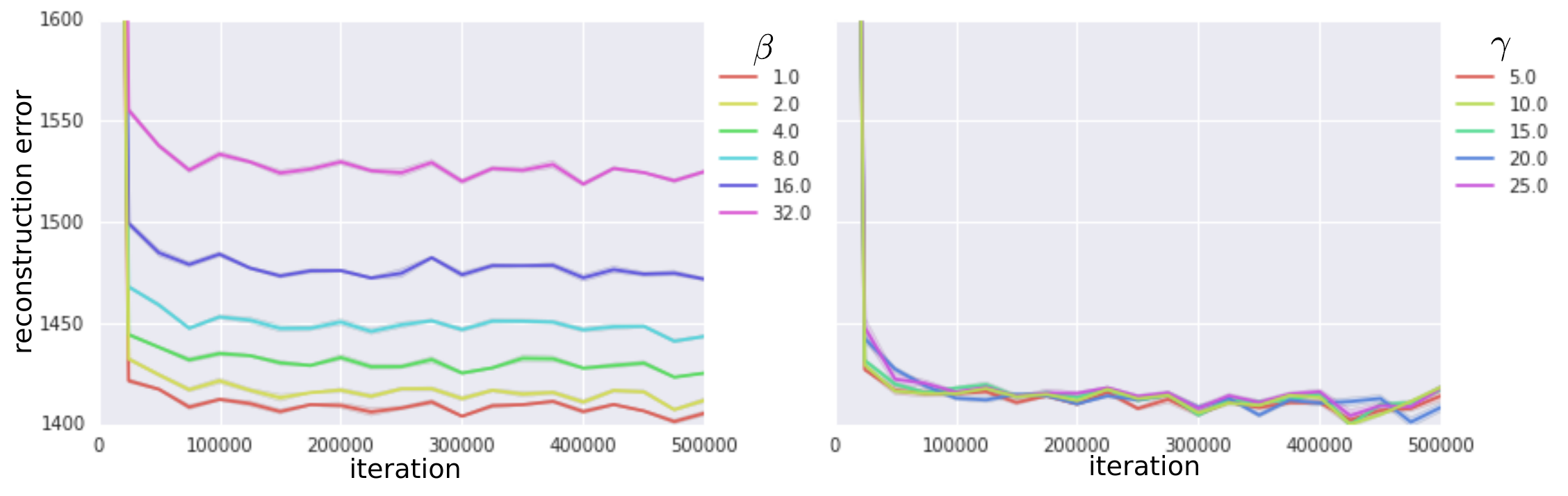}
  \vskip -0.1in
  \caption{Plots of reconstruction error of $\beta$-VAE (left) and FactorVAE (right) for different values of $\beta$ and $\gamma$ on 3D Faces data over 5 random seeds.}\label{fig:faces3d_rec_err}
\vskip -0.1in
\end{figure}

\begin{figure}[h!]
  \centering
  \includegraphics[width=0.8\columnwidth]{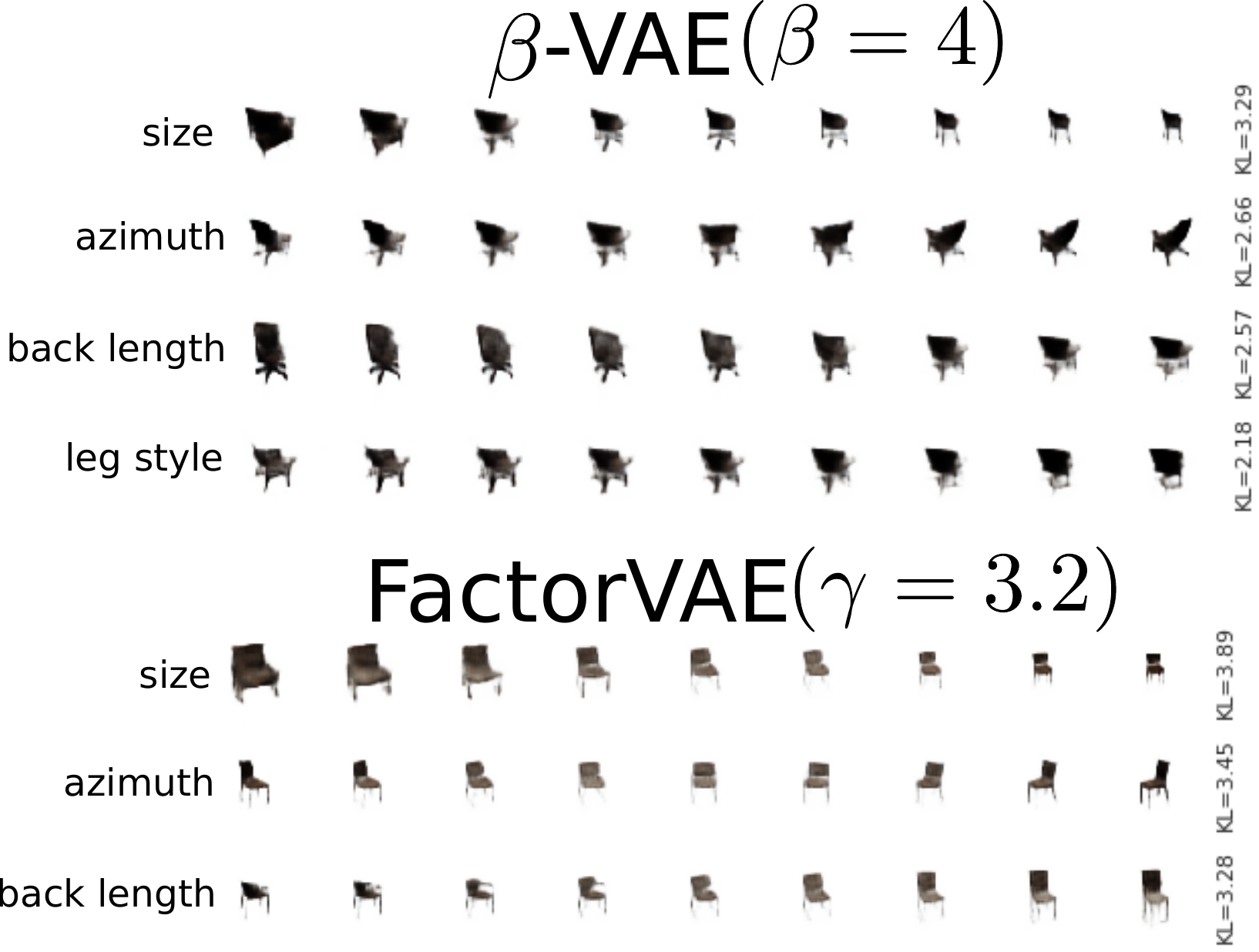}
  \vskip -0.1in
  \caption{$\beta$-VAE and FactorVAE latent traversals across each latent dimension sorted by KL on 3D Chairs, with annotations of the factor of variation corresponding to each latent unit.}\label{fig:chairs_latent_traversals}
\vskip -0.1in
\end{figure}

\begin{figure}[h!]
  \centering
  \includegraphics[width=\columnwidth]{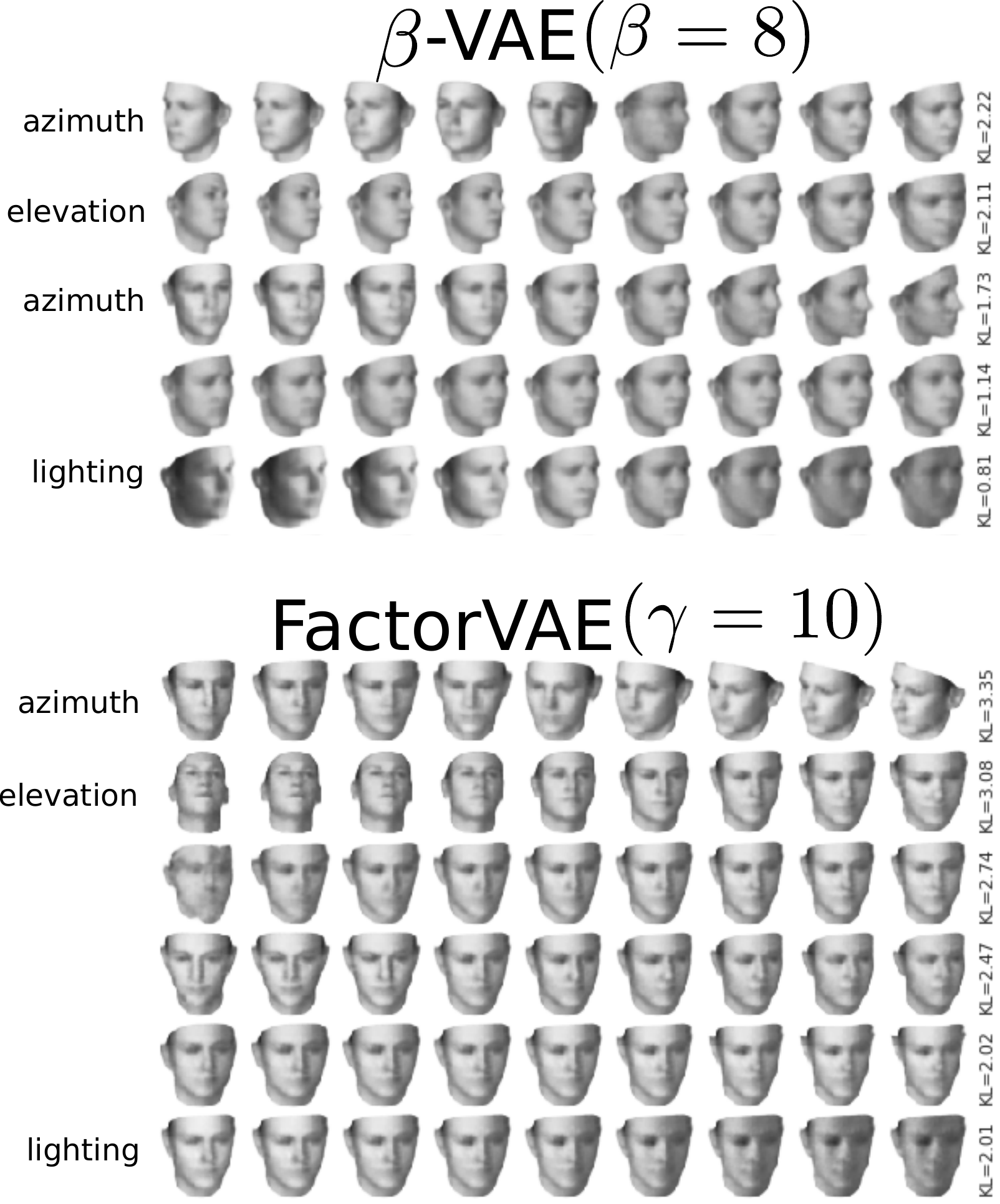}
  \vskip -0.1in
  \caption{Same as \fig{fig:chairs_latent_traversals} but for 3D Faces.}\label{fig:faces3d_latent_traversals}
\vskip -0.1in
\end{figure}

\begin{figure}[h!]
  \centering
  \includegraphics[width=0.9\columnwidth]{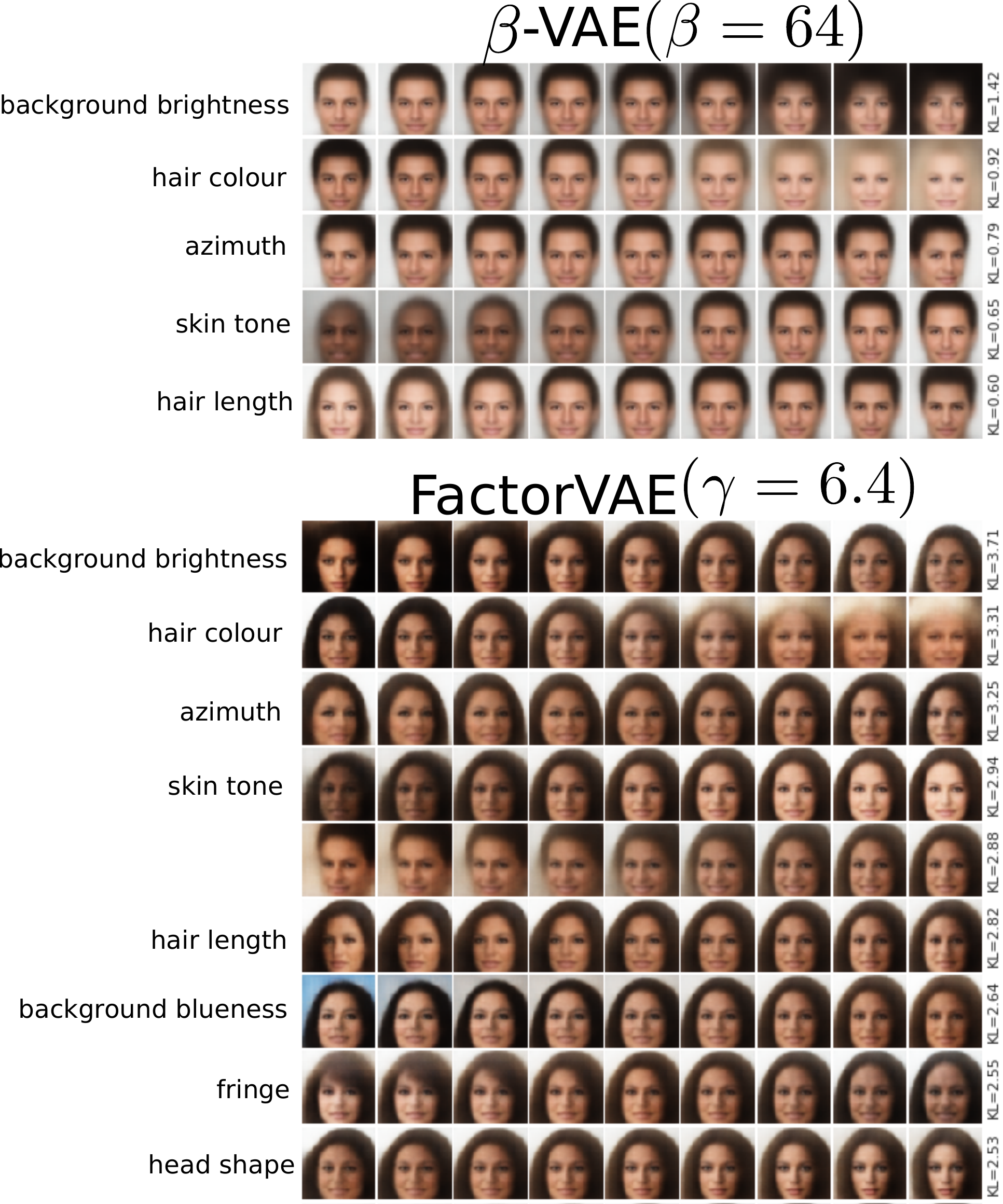}
  \vskip -0.1in
  \caption{Same as Figure \ref{fig:chairs_latent_traversals} but for CelebA.}\label{fig:celeba_latent_traversals}
\vskip -0.1in
\end{figure}

\section{Conclusion and Discussion}
We have introduced FactorVAE, a novel method for disentangling that achieves better disentanglement scores than $\beta$-VAE on the 2D Shapes and 3D Shapes data sets for the same reconstruction quality. Moreover, we have identified weaknesses of the commonly used disentanglement metric of \citet{higgins2016beta}, and proposed an alternative metric that is conceptually simpler, is free of hyperparameters, and avoids the failure mode of the former. Finally, we have performed an experimental evaluation of disentangling for the VAE-based methods and InfoWGAN-GP, a more stable variant of InfoGAN, and identified its weaknesses relative to the VAE-based methods.

One of the limitations of our approach is that low Total Correlation is necessary but not sufficient for disentangling of independent factors of variation. For example, if all but one of the latent dimensions were to collapse to the prior, the TC would be 0 but the representation would not be disentangled. Our disentanglement metric also requires us to be able to generate samples holding one factor fixed, which may not always be possible, for example when our training set does not cover all possible combinations of factors. The metric is also unsuitable for data with non-independent factors of variation.

For future work, we would like to use discrete latent variables to model discrete factors of variation and investigate how to reliably capture combinations of discrete and continuous factors using discrete and continuous latents.

\section*{Acknowledgements}
We thank Chris Burgess and Nick Watters for providing the data sets and helping to set them up, and thank Guillaume Desjardins, Sergey Bartunov, Mihaela Rosca, Irina Higgins and Yee Whye Teh for helpful discussions.



\bibliography{refs}
\bibliographystyle{icml2018}


\clearpage

\appendix

\section*{Appendix}

\section{Experimental Details for FactorVAE and $\beta$-VAE}
\label{apd:exp}

We use a Convolutional Neural Network for the encoder, a Deconvolutional Neural Network for the decoder and a Multi-Layer Perceptron (MLP) with for the discriminator in FactorVAE for experiments on all data sets. We use [0,1] normalised data as targets for the mean of a Bernoulli distribution, using negative cross-entropy for $\log p(x|z)$ and Adam optimiser \cite{kingma2014adam} with learning rate $10^{-4}$, $\beta_1=0.9, \beta_2=0.999$ for the VAE updates, as in \citet{higgins2016beta}. We also use Adam for the discriminator updates with $\beta_1=0.5, \beta_2=0.9$ and a learning rate tuned from $\{10^{-4},10^{-5}\}$. We use $10^{-4}$ for 2D Shapes and 3D Faces, and $10^{-5}$ for 3D Shapes, 3D Chairs and CelebA. The encoder outputs parameters for the mean and log-variance of Gaussian $q(z|x)$, and the decoder outputs logits for each entry of the image. We use the same encoder/decoder architecture for $\beta$-VAE and FactorVAE, shown in Tables \ref{tab:2dshapes_enc_dec}, \ref{tab:3dshapes_celeba_chairs_enc_dec}, and \ref{tab:3dfaces_enc_dec}. We use the same 6 layer MLP discriminator with 1000 hidden units per layer and leaky ReLU (lReLU) non-linearity, that outputs 2 logits in all FactorVAE experiments. We noticed that smaller discriminator architectures work fine, but noticed small improvements up to 6 hidden layers and 1000 hidden units per layer. Note that scaling the discriminator learning rate is not equivalent to scaling $\gamma$, since $\gamma$ does not affect the discriminator loss. See Algorithm \ref{alg:factor_vae} for details of FactorVAE updates. We train for $3\times 10^5$ iterations on 2D Shapes, $5\times 10^5$ iterations on 3D Shapes, and $10^6$ iterations on Chairs, 3D Faces and CelebA. We use a batch size of 64 for all data sets.

\begin{table}[h!]
\caption{Encoder and Decoder architecture for 2D Shapes data.}
\label{tab:2dshapes_enc_dec}
\vskip -0.1in
\begin{center}
\begin{small}
\resizebox{\columnwidth}{!}{
\begin{tabular}{|l|l|}
\toprule
\textbf{Encoder} & \textbf{Decoder} \\
\midrule
Input $64 \times 64$ binary image & Input $\in \mathbb{R}^{10}$ \\
\midrule
$4 \times 4$ conv. 32 ReLU. stride 2 & FC. 128 ReLU. \\
\midrule
$4 \times 4$ conv. 32 ReLU. stride 2 & FC. $4 \times 4 \times 64$ ReLU. \\
\midrule
$4 \times 4$ conv. 64 ReLU. stride 2 & $4 \times 4$ upconv. 64 ReLU. stride 2 \\
\midrule
$4 \times 4$ conv. 64 ReLU. stride 2 & $4 \times 4$ upconv. 32 ReLU. stride 2 \\
\midrule
FC. 128. FC. $2 \times 10$. & $4 \times 4$ upconv. 32 ReLU. stride 2\\
\midrule
& $4 \times 4$ upconv. 1. stride 2 \\
\bottomrule
\end{tabular}
}
\end{small}
\end{center}

\caption{Encoder and Decoder architecture for 3D Shapes, CelebA, Chairs data.}
\label{tab:3dshapes_celeba_chairs_enc_dec}
\vskip -0.1in
\begin{center}
\begin{small}
\resizebox{\columnwidth}{!}{
\begin{tabular}{|l|l|}
\toprule
\textbf{Encoder} & \textbf{Decoder} \\
\midrule
Input $64 \times 64 \times 3$ RGB image & Input $\in$ $\mathbb{R}^{6}$ (3D Shapes) $\mathbb{R}^{10}$ (CelebA, Chairs) \\
\midrule
$4 \times 4$ conv. 32 ReLU. stride 2 & FC. 256 ReLU. \\
\midrule
$4 \times 4$ conv. 32 ReLU. stride 2 & FC. $4 \times 4 \times 64$ ReLU. \\
\midrule
$4 \times 4$ conv. 64 ReLU. stride 2 & $4 \times 4$ upconv. 64 ReLU. stride 2 \\
\midrule
$4 \times 4$ conv. 64 ReLU. stride 2 & $4 \times 4$ upconv. 32 ReLU. stride 2 \\
\midrule
FC. 256. FC. $2 \times 10$. & $4 \times 4$ upconv. 32 ReLU. stride 2\\
\midrule
& $4 \times 4$ upconv. 3. stride 2 \\
\bottomrule
\end{tabular}
}
\end{small}
\end{center}

\caption{Encoder and Decoder architecture for 3D Faces data.}
\label{tab:3dfaces_enc_dec}
\vskip -0.1in
\begin{center}
\begin{small}
\resizebox{\columnwidth}{!}{
\begin{tabular}{|l|l|}
\toprule
\textbf{Encoder} & \textbf{Decoder} \\
\midrule
Input $64 \times 64$ greyscale image & Input $\in \mathbb{R}^{10}$ \\
\midrule
$4 \times 4$ conv. 32 ReLU. stride 2 & FC. 256 ReLU. \\
\midrule
$4 \times 4$ conv. 32 ReLU. stride 2 & FC. $4 \times 4 \times 64$ ReLU. \\
\midrule
$4 \times 4$ conv. 64 ReLU. stride 2 & $4 \times 4$ upconv. 64 ReLU. stride 2 \\
\midrule
$4 \times 4$ conv. 64 ReLU. stride 2 & $4 \times 4$ upconv. 32 ReLU. stride 2 \\
\midrule
FC. 256. FC. $2 \times 10$. & $4 \times 4$ upconv. 32 ReLU. stride 2\\
\midrule
& $4 \times 4$ upconv. 1. stride 2 \\
\bottomrule
\end{tabular}
}
\end{small}
\end{center}
\end{table}

\section{Details for the Disentanglement Metrics}
\label{apd:metric}

\begin{figure}[h!]
\vskip -0.1in
  \centering
  \includegraphics[width=\columnwidth]{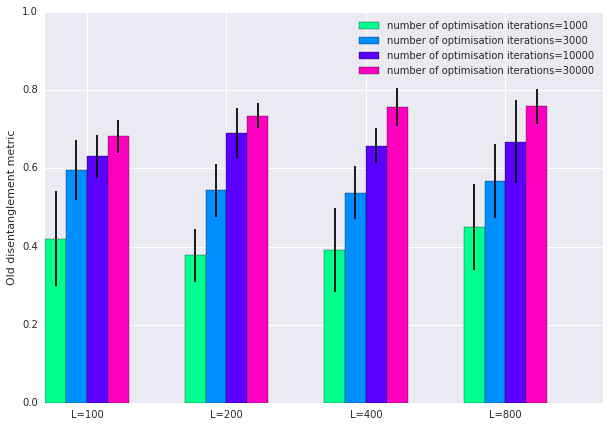}
  \vskip -0.1in
  \caption{Mean and standard deviation of metric in \citet{higgins2016beta} across 10 random seeds for varying $L$ and number of Adagrad optimiser iterations (batch size 10). The number of points used for evaluation after optimisation is fixed to 800. These were all evaluated on a fixed, randomly chosen $\beta$-VAE model that was trained to convergence on the 2D Shapes data.}\label{fig:disent_sensitivity}
\vskip -0.1in
\end{figure}

\begin{figure}[h!]
\vskip -0.1in
  \centering
  \includegraphics[width=\columnwidth]{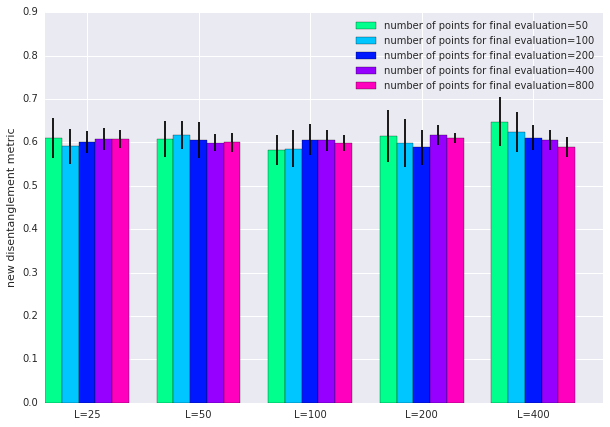}
  \vskip -0.1in
  \caption{Mean and standard deviation of our metric across 10 random seeds for varying $L$ and number of points used for evaluation. These were all evaluated on a fixed, randomly chosen $\beta$-VAE model that was trained to convergence on the 2D Shapes data.}\label{fig:disent_new_sensitivity}
\vskip -0.1in
\end{figure}

We performed a sensitivity analysis of each metric with respect to its hyperparameters (c.f. Figure \ref{fig:metric}). In Figures \ref{fig:disent_sensitivity}, we show that the metric in \citet{higgins2016beta} is very sensitive to number of iterations of the Adagrad \cite{duchi2011adaptive} optimiser with learning rate 0.01 (used in \citet{higgins2016beta}), and constantly improves with more iterations. This suggests that one might want to use less noisy multi-class logistic regression solvers than gradient-descent based methods. The number of data points used to evaluate the metric after optimisation did not seem to help reduce variance beyond 800. So in our experiments, we use $L=200$ and 10000 iterations, with a batch size of 10 per iteration of training the linear classifier, and use a batch of size 800 to evaluate the metric at the end of training. Each evaluation of this metric took around 30 minutes on a single GPU, hence we could not afford to train for more iterations.

For our disentanglement metric, we first prune out all latent dimensions that have collapsed to the prior ($q(z_j|x)=p(z_j)$). Then we just use the surviving dimensions for the majority vote. From the sensitivity analysis our metric in Figure \ref{fig:disent_new_sensitivity}, we observe that our metric is much less sensitive to hyperparameters than the metric in \citet{higgins2016beta}. We use $L=100$ and take the majority vote classifier from 800 votes. This only takes a few seconds on a single GPU. The majority vote classifier $C$ works as follows: suppose we are given data $(a_i,b_i)_{i=1}^M, a_i \in \{1,\ldots,D\}, b_i \in \{1,\ldots,K\}$ (so $M=500$). Then for $j \in \{1,\ldots,D\}$, let $V_{jk}=\sum_{i=1}^M \mathbb{I}(a_i=j,b_i=k)$. Then the majority vote classifier is defined to be $C(j)= \argmax_{k} V_{jk}$.

Note that $D$, the dimensionality of the latents, does not affect the metric; for a classifier that chooses at random, the accuracy is $1/K$, independent of $D$.

For discrete latent variables, we use Gini's definition of empirical variance:
\begin{equation}
\reallywidehat{Var}(x) = \frac{1}{2N(N-1)} \sum_{i,j=1}^N d(x_i,x_j)
\end{equation}
for $x=[x_1,\ldots,x_N] \in \mathbb{R}^N$, $d(x_i,x_j)=1$ if $x_i \neq x_j$ and 0 if $x_i=x_j$. Note that this is equal to empirical variance for continuous variables when $d(x_i,x_j)=(x_i-x_j)^2$.

\section{KL Decomposition}
\label{apd:kl_decomp}
The KL term in the VAE objective decomposes as follows \cite{makhzani2017pixelgan}:
\begin{lemma}
$\mathbb{E}_{p_{data}(x)}[KL(q(z|x)||p(z))] = I_q(x;z) + KL(q(z)||p(z))$ where $q(x,z) = p_{data}(x)q(z|x)$.
\end{lemma}
\begin{proof}
\begin{align*}
&\mathbb{E}_{p_{data}(x)}  [KL(q(z|x)||p(z))]  \\ 
& =  \mathbb{E}_{p_{data}(x)} \mathbb{E}_{q(z|x)}\bigg[\log \frac{q(z|x)}{p(z)} \bigg]  \\
&=\mathbb{E}_{p_{data}(x)} \mathbb{E}_{q(z|x)}\bigg[\log \frac{q(z|x)}{q(z)}\frac{q(z)}{p(z)} \bigg]  \\
&=\mathbb{E}_{p_{data}(x)} \mathbb{E}_{q(z|x)}\bigg[\log \frac{q(z|x)}{q(z)}+ \log \frac{q(z)}{p(z)} \bigg]  \\
&=\mathbb{E}_{p_{data}(x)}[KL(q(z|x)||q(z))] + \mathbb{E}_{q(x,z)}\bigg[ \log \frac{q(z)}{p(z)} \bigg]  \\
&=I_q(x;z) + \mathbb{E}_{q(z)} \bigg[ \log \frac{q(z)}{p(z)}\bigg]  \\
&=I_q(x;z) +  KL(q(z)||p(z))
\end{align*}
\end{proof}

\begin{remark}
Note that this decomposition is equivalent to that in \citet{hoffman2016elbo}, written as follows:
$\mathbb{E}_{p_{data}(x)}[KL(q(z|x)||p(z))]= I_r(i;z) + KL(q(z)||p(z))$
where $r(i,z) = \frac{1}{N}q(z|x^{(i)})$, hence $r(z|i) = q(z|x^{(i)})$, $r(z) = \frac{1}{N} \sum_{i=1}^N q(z|x^{(i)}) = q(z)$.
\end{remark}
\begin{proof}
\begin{align*}
I_r(i;z) &= \mathbb{E}_{r(i)}[KL(r(z|i)||r(z))] \\
&=\frac{1}{N} \sum_{i=1}^N KL(q(z|x^{(i)})||q(z)) \\
&= \mathbb{E}_{p_{data}(x)} [KL(q(z|x)||q(z))] \\
&= I_q(x;z)
\end{align*}
\end{proof}

\section{Using a Batch Estimate of $q(z)$ for Estimating TC}
\label{apd:batch_estimate_q}

We have also tried using a batch estimate for the density $q(z)$, thus optimising this estimate of the TC directly instead of having a discriminator and using the density ratio trick. In other words, we tried $q(z) \approx \hat{q}(z) = \frac{1}{|\mathcal{B}|} \sum_{i\in \mathcal{B}} q(z|x^{(i)})$, and using the estimate:
\begin{align}
KL(q(z)||\prod_j q(z_j)) &= \mathbb{E}_{q(z)} \bigg[ \log \frac{q(z)}{\prod_j q(z_j)}\bigg] \nonumber \\
&\approx \mathbb{E}_{q(z)} \bigg[ \log \frac{\hat{q}(z)}{\prod_j \hat{q}(z_j)}\bigg]
\end{align}
Note that:
\begin{align}
\mathbb{E}_{q(z)} \bigg[ \log \frac{\hat{q}(z)}{\prod_j \hat{q}(z_j)}\bigg]& \nonumber \\
\approx \frac{1}{H} \sum_{h=1}^H \bigg[& \log \frac{1}{|\mathcal{B}|} \sum_{i \in \mathcal{B}} \prod_{j=1}^D q(z^{(h)}_j|x^{(i)}) \nonumber \\
& - \log \prod_{j=1}^D \frac{1}{|\mathcal{B}|} \sum_{i \in \mathcal{B}} q(z^{(h)}_j|x^{(i)}) \bigg]
\end{align}
for $z^{(h)} \overset{iid}{\sim} q(z)$. However while experimenting on 2D Shapes, we observed that the value of $\log q(z^{(h)})$ becomes very small (negative with high absolute value) for latent dimension $d \geq 2$ during training, because $\hat{q}(z)$ is not a good enough approximation to $q(z)$ unless $\mathcal{B}$ is very big. As training progresses for the VAE, the variance of Gaussians $q(z|x^{(i)})$ becomes smaller and smaller, so they do not overlap too much in higher dimensions. Hence we get $z^{(h)} \sim q(z)$ that land on the tails of $\hat{q}(z)=\frac{1}{|\mathcal{B}|} \sum_{i\in \mathcal{B}} q(z|x^{(i)})$, giving worryingly small values of $\log \hat{q}(z^{(h)})$. On the other hand $\prod_j \hat{q}(z_j^{(h)})$, a mixture of $|\mathcal{B}|^d$ Gaussians hence of much higher entropy, gives much more stable values of $\log \prod_j \hat{q}(z_j^{(h)})$. From Figure \ref{fig:batch_estimate_q}, we can see that even with $\mathcal{B}$ as big as 10,000, we get negative values for the estimate of TC, which is a KL divergence and hence should be non-negative, hence this method of using a batch estimate for $q(z)$ does not work. A fix is to use samples from $\hat{q}(z)$ instead of $q(z)$, but this seemed to give a similar reconstruction-disentanglement trade-off to $\beta$-VAE. Very recently, work from \cite{chen2018isolating} has shown that disentangling can be improved by using samples from $\hat{q}(z)$.

\begin{figure}[h!]
\vskip -0.1in
  \centering
  \includegraphics[width=\columnwidth]{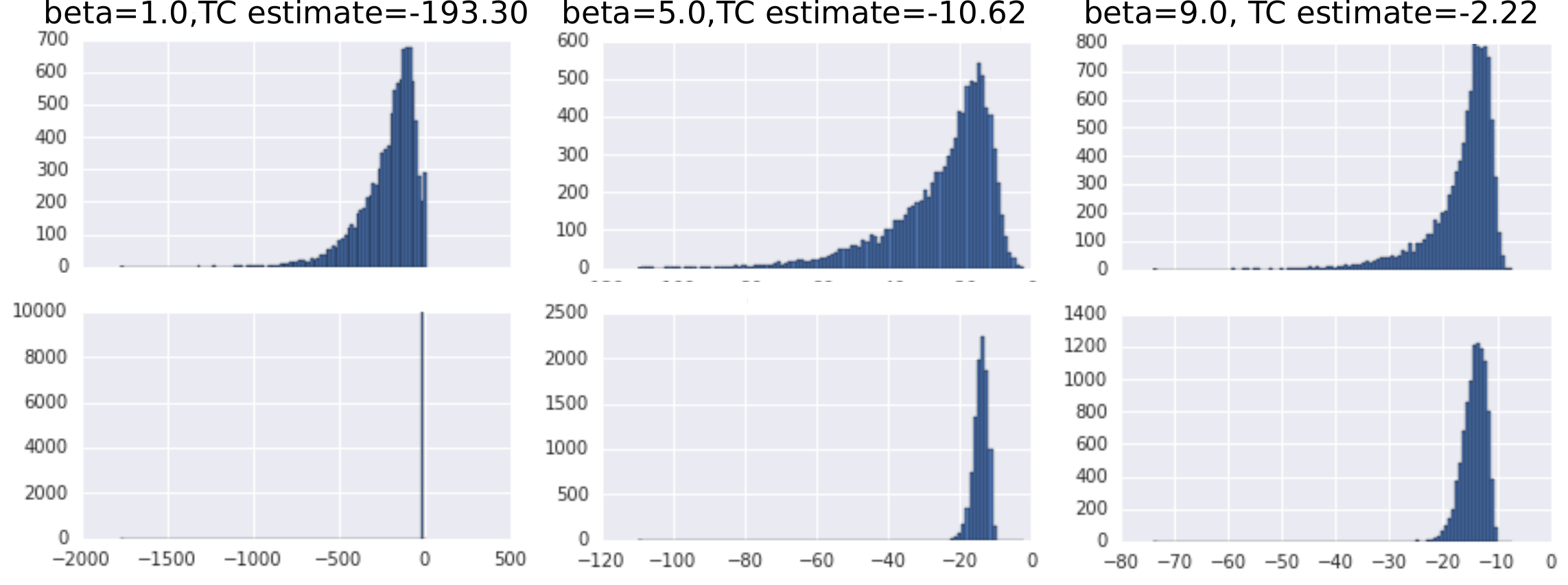}
  \vskip -0.1in
  \caption{Histogram of $\log \hat{q}(z^{(h)})$ (top) and $\prod_{j=1}^d \hat{q}(z_j^{(h)})$ (bottom) for $z^{(h)} \overset{iid}{\sim} q(z)$ with $|\mathcal{B}|=10000$, $d=10$. The columns correspond to values of $\beta=1,5,9$ for training $\beta$-VAE. In the title of each histogram, there is an estimate of TC based on the samples of $z^{(h)}$.}\label{fig:batch_estimate_q}
\vskip -0.1in
\end{figure}

\section{Log Marginal Likelihood and Samples}
\label{apd:generative}

\begin{figure}[h!]
\vskip -0.1in
  \centering
  \includegraphics[width=\columnwidth]{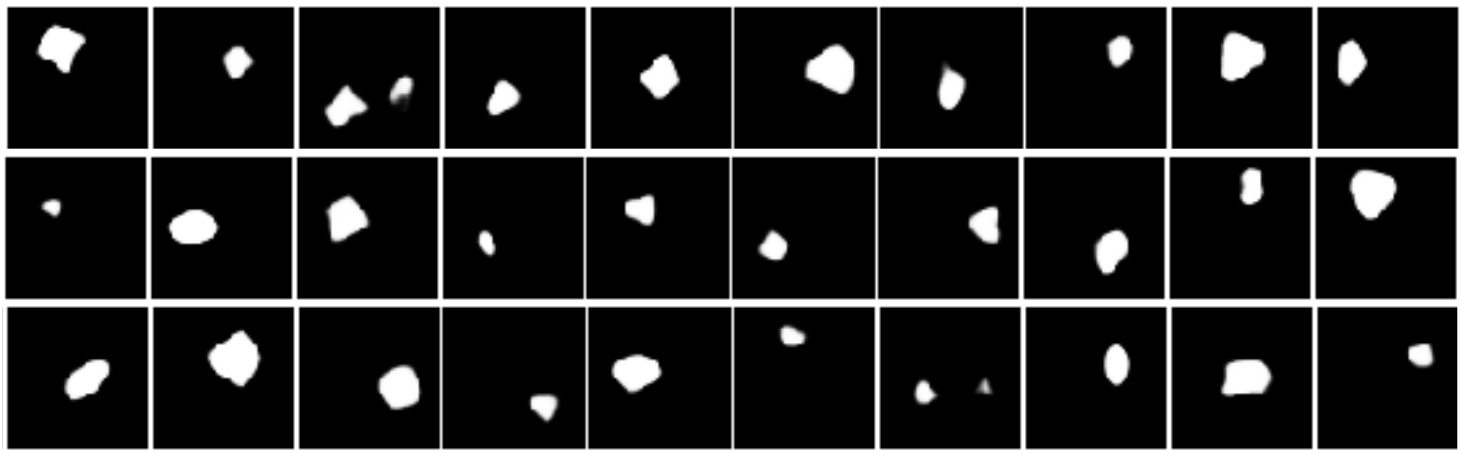}
  \vskip -0.1in
  \caption{Randomly chosen samples from the best performing (in terms of disentanglement) $\beta$-VAE generative model ($\beta=4$).}\label{fig:beta_vae_samples}
\end{figure}

\begin{figure}[h!]
  \centering
  \includegraphics[width=\columnwidth]{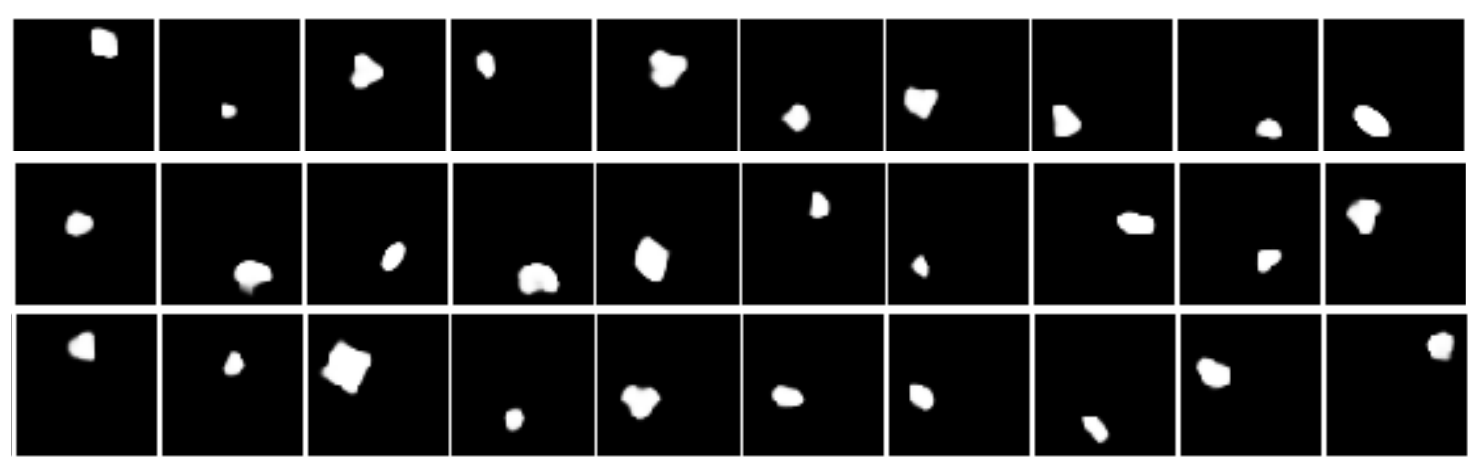}
  \vskip -0.1in
  \caption{Randomly chosen samples from the best performing (in terms of disentanglement) FactorVAE generative model ($\gamma=35$).}\label{fig:factor_vae_samples}
\end{figure}

\begin{figure}[h!]
  \centering
  \includegraphics[width=\columnwidth]{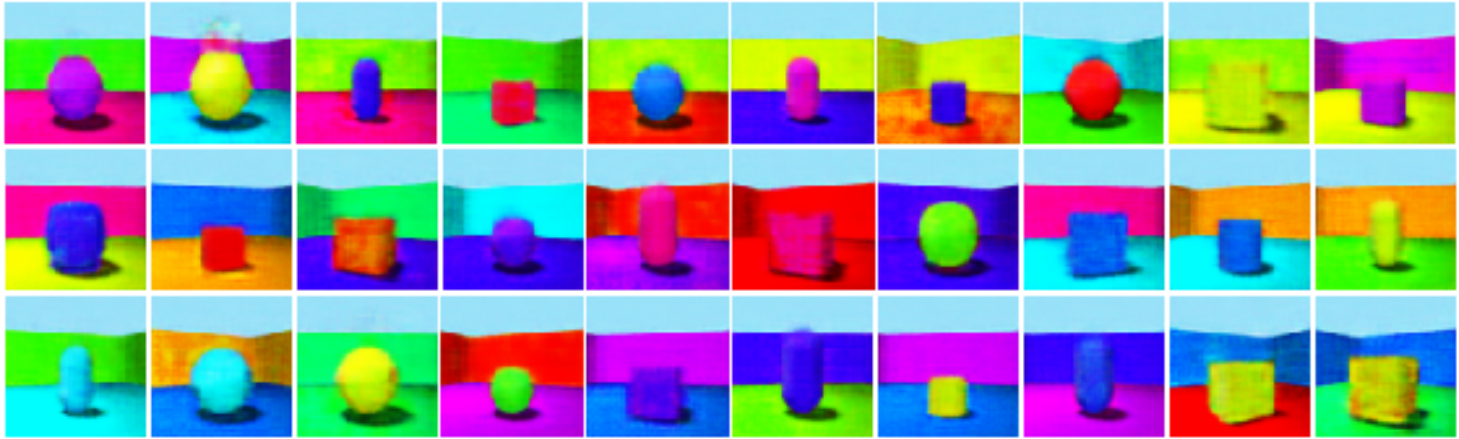}
  \vskip -0.1in
  \caption{Randomly chosen samples from the best performing (in terms of disentanglement) $\beta$-VAE generative model ($\beta=32$).}\label{fig:beta_vae_3d_samples}
\end{figure}

\begin{figure}[h!]
  \centering
  \includegraphics[width=\columnwidth]{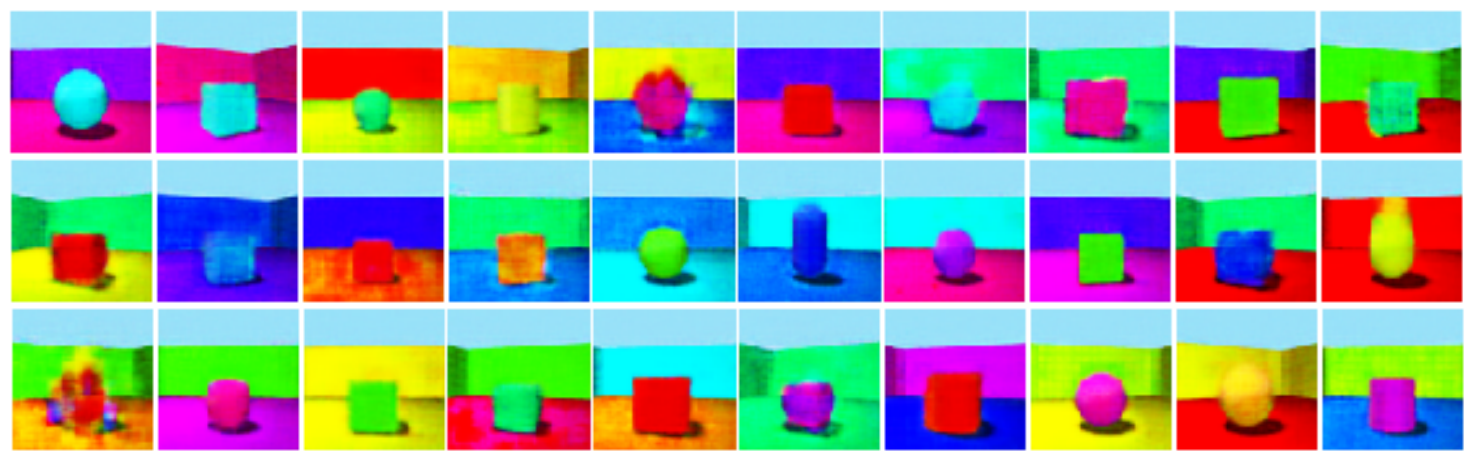}
  \vskip -0.1in
  \caption{Randomly chosen samples from the best performing (in terms of disentanglement) FactorVAE generative model ($\gamma=6$).}\label{fig:factor_vae_3d_samples}
\end{figure}

We give the log marginal likelihood of each of the best performing $\beta$-VAE and FactorVAE models (in terms of disentanglement) for both the 2D Shapes and 3D Shapes data sets along with samples from the generative model. Since the log marginal likelihood is intractable, we report the \textit{Importance-Weighted Autoencoder} (IWAE) bound with 5000 particles, in line with standard practice in the generative modelling literature.

In Figures \ref{fig:beta_vae_samples} and \ref{fig:factor_vae_samples}, the samples for FactorVAE are arguably more representative of the data set than those of $\beta$-VAE. For example $\beta$-VAE has occasional samples with two separate shapes in the same image (Figure \ref{fig:beta_vae_samples}). The log marginal likelihood for the best performing $\beta$-VAE ($\beta=4$) is -46.1, whereas for FactorVAE it is -51.9 ($\gamma=35$) (a randomly chosen VAE run gives -43.3). So on 2D Shapes, FactorVAE gives better samples but worse log marginal likelihood.

In Figures \ref{fig:beta_vae_3d_samples} and \ref{fig:factor_vae_3d_samples}, the samples for $\beta$-VAE appear more coherent than those for FactorVAE. However the log marginal likelihood for $\beta$-VAE ($\beta=32$) is -3534, whereas for FactorVAE it is -3520 ($\gamma=6$) (a randomly chosen VAE run gives -3517). So on 3D Shapes, FactorVAE gives worse samples but better log marginal likelihood.

In general, if one seeks to learn a generative model with a disentangled latent space, it would make sense to choose the model with the lowest value of $\beta$ or $\gamma$ among those with similarly high disentanglement performance. 

\section{Losses and Experiments for other related Methods}
\label{apd:related}
The Adversarial Autoencoder (AAE) \cite{makhzani2015adversarial} uses the following objective
\begin{align}
\frac{1}{N}\sum_{i=1}^N \Big[ \mathbb{E}_{q(z|x^{(i)})} [\log p(x^{(i)}|z)] \Big] - KL(q(z)||p(z)), 
\end{align}
utilising the density ratio trick to estimate the KL term.

Information Dropout \cite{achille2018information} uses the objective
\begin{equation}
\frac{1}{N}\sum_{i=1}^N  \mathbb{E}_{q(z|x^{(i)})}[\log p(x^{(i)}|z)] - \beta KL(q(z|x^{(i)})||q(z)).
\end{equation}
The following objective is also considered in the paper but is dismissed as intractable: 
\begin{align}
\frac{1}{N}\sum_{i=1}^N \Big[ \mathbb{E}_{q(z|x^{(i)})}[\log p(x^{(i)}|z)] - & \beta  KL(q(z|x^{(i)})||q(z)) \Big] \nonumber \\ 
 - & \gamma KL(q(z)||\prod_{j=1}^d q(z_j))
\end{align}
Note that it is similar to the FactorVAE objective (which has $\beta=1$), but with $p(z)$ in the first KL term replaced with $q(z)$.

DIP-VAE \cite{kumar2017variational} uses the VAE objective with an additional penalty on how much the covariance of $q(z)$ deviates from the identity matrix, either using the law of total covariance $Cov_{q(z)}[z] = \mathbb{E}_{p_{data}(x)} Cov_{q(z|x)}[z] + Cov_{p_{data}(x)}(\mathbb{E}_{q(z|x)}[z])$ (DIP-VAE I):
\begin{align}
\frac{1}{N}\sum_{i=1}^N  & \Big[ \mathbb{E}_{q(z|x^{(i)})}[\log p(x^{(i)}|z)] - KL(q(z|x^{(i)})||p(z)) \Big] \nonumber \\
& - \lambda_{od} \sum_{i\neq j} [Cov_{p_{data}(x)}[\mu(x)]]_{ij}^2 \nonumber \\
& - \lambda_{d} \sum_{i}([Cov_{p_{data}(x)}[\mu(x)]]_{ii} - 1)^2
\end{align}
where $\mu(x)=mean(q(z|x))$, or directly (DIP-VAE II):
\begin{align}
\frac{1}{N}\sum_{i=1}^N  & \Big[\mathbb{E}_{q(z|x^{(i)})}[\log p(x^{(i)}|z)] - KL(q(z|x^{(i)})||p(z)) \Big] \nonumber \\
& - \lambda_{od} \sum_{i\neq j} [Cov_{q(z)}[z]]_{ij}^2 \nonumber \\
& - \lambda_{d} \sum_{i}([Cov_{q(z)}[z]]_{ii} - 1)^2
\end{align}

One could argue that during training of FactorVAE, $\prod_j q(z_j)$ will be similar to $p(z)$, assuming the prior is factorial, due to the $KL(q(z|x)||p(z))$ term in the objective. Hence we also investigate a modified FactorVAE objective that replaces $\prod_j q(z_j)$ with $p(z)$:
\begin{align} \label{eq:modified_factor_vae}
\frac{1}{N}\sum_{i=1}^N  \Big[ \mathbb{E}_{q(z|x^{(i)})}[\log p(x^{(i)}|z)] - & KL(q(z|x^{(i)})||p(z)) \Big] \nonumber \\ 
 - & \gamma KL(q(z)||p(z))
\end{align}
However as shown in Figure \ref{fig:shapes2d_histograms} of Appendix \ref{apd:further_results}, the histograms of samples from the marginals are clearly quite different from the the prior for FactorVAE.

\begin{figure}[h!]
\vskip -0.1in
  \centering
  \includegraphics[width=\columnwidth]{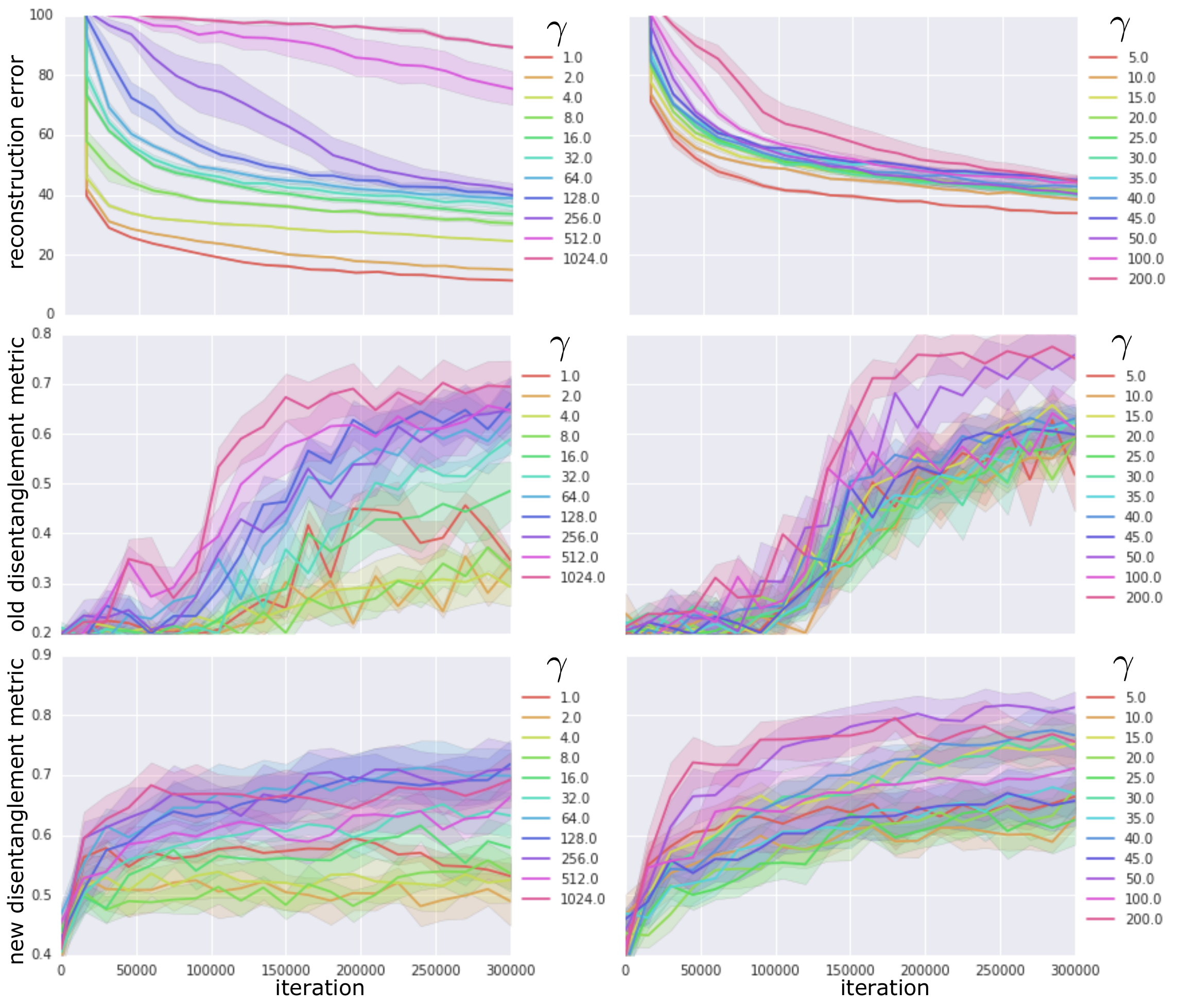}
  \vskip -0.1in
  \caption{Same as Figure \ref{fig:shapes2d_plots} but for AAE (left) and a variant of FactorVAE (\eqn{eq:modified_factor_vae}.}\label{fig:aae_modified_factor_vae_plots}
\vskip -0.1in
\end{figure}

Moreover we show experimental results for AAE (adding a $\gamma$ coefficient in front of the $KL(q(z)||p(z))$ term of the objective and tuning it) and the variant of FactorVAE (\eqn{eq:modified_factor_vae}) on the 2D Shapes data. From Figure \ref{fig:aae_modified_factor_vae_plots}, we see that the disentanglement performance for both are somewhat lower than that for FactorVAE. This difference could be explained as a benefit of directly encouraging $q(z)$ to be factorised (FactorVAE) instead of encouraging it to approach an arbitrarily chosen factorised prior $p(z)=\mathcal{N}(0,I)$ (AAE, \eqn{eq:modified_factor_vae}). Information Dropout and DIP-VAE did not have enough experimental details in the paper nor publicly available code to have their results reproduced and compared against.

\section{InfoGAN and InfoWGAN-GP} \label{apd:infogan_overview}
We give an overview of InfoGAN \cite{chen2016infogan} and InfoWGAN-GP, its counterpart using Wasserstein distance and gradient penalty. InfoGAN uses latents $z=(c,\epsilon)$ where $c$ models semantically meaningful codes and $\epsilon$ models incompressible noise. The generative model is defined by a generator $G$ with the process: $c \sim p(c), \epsilon \sim p(\epsilon), z=(c,\epsilon)$, $x=G(z)$. i.e. $p(z)=p(c)p(\epsilon)$. GANs are defined as a minimax game on some objective $V(D,G)$, where $D$ is either a discriminator (\eg for the original GAN \cite{goodfellow2014generative}) that outputs log probabilities for binary classification, or a critic (\eg  for Wasserstein-GAN \cite{arjovsky2017wasserstein}) that outputs a real-valued scalar. InfoGAN defines an extra encoding distribution $Q(c|x)$ that is used to define an extra penalty:
\begin{equation}
L(G,Q) = \mathbb{E}_{p(c)}\mathbb{E}_{p(\epsilon)}[\log Q(c|G(c,\epsilon))]
\end{equation}
that is added to the GAN objective. Hence InfoGAN is the following minimax game on the parameters of neural nets $D,G,Q$:
\begin{equation}
\min_{G,Q} \max_{D} V_I(D,G,Q) = \min_{G,Q} \max_{D} V(D,G) - \lambda L(G,Q)
\end{equation}
$L$ can be interpreted as a variational lower bound to $I(c; G(c,\epsilon))$, with equality at $Q = \argmin_Q V_I(D,G,Q)$. i.e. $L$ encourages the codes to be more informative about the image. From the definition of $L$, it can also be seen as the reconstruction error of codes in the latent space. The original InfoGAN defines: 
\begin{equation}
V(D,G)= \mathbb{E}_{p_{data}(x)}[D(x)] - \mathbb{E}_{p(z)}[D(G(z))]
\end{equation}
same as the original GAN objective where $D$ outputs log probabilities. However as we'll show in Appendix \ref{apd:infogan} this has known instability issues in training. So it is natural to try replacing this with the more stable WGAN-GP \cite{gulrajani2017improved} objective:
\begin{align}
V(D,G)= & \mathbb{E}_{p_{data}(x)}[D(x)] - \mathbb{E}_{p(z)}[D(G(z))] \nonumber \\
& + \eta(||\nabla_{x}D(x)|_{x=\hat{x}}||_2 - 1)^2
\end{align}
for $\hat{x}=\pi x_r + (1-\pi) x_f$ with $\pi\sim U[0,1]$, $x_r\sim p_{data}(x)$, $z_f\sim p(z)$, $x_f=$\texttt{stop\_gradient}$(G(z_f))$ and with a new $\hat{x}$ for each iteration of optimisation. Thus we obtain InfoWGAN-GP.

\section{Empirical Study of InfoGAN and InfoWGAN-GP} \label{apd:infogan}
To begin with, we implemented InfoGAN and InfoWGAN-GP on MNIST using the hyperparameters given in \citet{chen2016infogan} to better understand its behaviour, using 1 categorical code with 10 categories, 2 continuous codes, and 62 noise variables. We use priors $p(c_j)=U[-1,1]$ for the continuous codes, $p(c_j)=\frac{1}{J}$ for categorical codes with $J$ categories, and $p(\epsilon_j)=\mathcal{N}(0,1)$ for the noise variables. For 2D Shapes data we use 1 categorical codes with 3 categories ($J=3$), 4 continuous codes, and 5 noise variables. The number of noise variables did not seem to have a noticeable effect on the experiment results. We use the Adam optimiser \cite{kingma2014adam} with $\beta_1=0.5, \beta_2=0.999$, and learning rate $10^{-3}$ for the generator updates and $10^{-4}$ for the discriminator updates. The detailed Discriminator/Encoder/Generator architecture are given in Tables \ref{tab:mnist_infogan} and \ref{tab:2dshapes_infogan}. The architecture for InfoWGAN-GP is the same as InfoGAN, except that we use no Batch Normalisation (batchnorm) \cite{ioffe2015batch} for the convolutions in the discriminator, and replace batchnorm with Layer Normalisation \cite{ba2016layer} in the fully connected layer that follows the convolutions as recommended in \cite{gulrajani2017improved}. We use gradient penalty coefficient $\eta=10$, again as recommended.

\begin{table}[h!]
\caption{InfoGAN architecture for MNIST data. 2 continuous codes, 1 categorical code with 10 categories, 62 noise variables.}
\label{tab:mnist_infogan}
\vskip -0.1in
\begin{center}
\begin{small}
\resizebox{\columnwidth}{!}{
\begin{tabular}{|l|l|}
\toprule
\textbf{discriminator} D / \textbf{encoder} Q & \textbf{generator} G \\
\midrule
Input $28 \times 28$ greyscale image & Input $\in \mathbb{R}^{74}$ \\
\midrule
$4 \times 4$ conv. 64 lReLU. stride 2 & FC. 1024 ReLU. batchnorm \\
\midrule
$4 \times 4$ conv. 128 lReLU. stride 2. batchnorm & FC. $7 \times 7 \times 128$ ReLU. batchnorm \\
\midrule
FC. 1024 lReLU. batchnorm & $4 \times 4$ upconv. 64 ReLU. stride 2. batchnorm \\
\midrule
FC. 1. output layer for D & $4 \times 4$ upconv. 1 Sigmoid. stride 2 \\
\midrule
FC. 128 lReLU. batchnorm. FC $2 \times 2 + 1 \times 10$  &  \\
\bottomrule
\end{tabular}
}
\end{small}
\end{center}

\caption{InfoGAN architecture for 2D Shapes data. 4 continuous codes, 1 categorical code with 3 categories, 5 noise variables.}
\label{tab:2dshapes_infogan}
\vskip -0.1in
\begin{center}
\begin{small}
\resizebox{\columnwidth}{!}{
\begin{tabular}{|l|l|}
\toprule
\textbf{discriminator} D / \textbf{encoder} Q & \textbf{generator} G \\
\midrule
Input $64 \times 64$ binary image & Input $\in \mathbb{R}^{12}$ \\
\midrule
$4 \times 4$ conv. 32 lReLU. stride 2 & FC. 128 ReLU. batchnorm \\
\midrule
$4 \times 4$ conv. 32 lReLU. stride 2. batchnorm & FC. $4 \times 4 \times 64$ ReLU. batchnorm \\
\midrule
$4 \times 4$ conv. 64 lReLU. stride 2. batchnorm & $4 \times 4$ upconv. 64 lReLU. stride 2. batchnorm \\
\midrule
$4 \times 4$ conv. 64 lReLU. stride 2. batchnorm & $4 \times 4$ upconv. 32 lReLU. stride 2. batchnorm \\
\midrule
FC. 128 lReLU. batchnorm & $4 \times 4$ upconv. 32 lReLU. stride 2. batchnorm \\
\midrule
FC. 1. output layer for D & $4 \times 4$ upconv. 1 Sigmoid. stride 2 \\
\midrule
FC. 128 lReLU. batchnorm. FC $4 \times 2 + 1 \times 3$ for Q  & \\
\bottomrule
\end{tabular}
}
\end{small}
\end{center}

\caption{Bigger InfoGAN architecture for 2D Shapes data. 4 continuous codes, 1 categorical code with 3 categories, 128 noise variables.}
\label{tab:2dshapes_infogan_big}
\vskip -0.1in
\begin{center}
\begin{small}
\resizebox{\columnwidth}{!}{
\begin{tabular}{|l|l|}
\toprule
\textbf{discriminator} D / \textbf{encoder} Q & \textbf{generator} G \\
\midrule
Input $64 \times 64$ binary image & Input $\in \mathbb{R}^{136}$ \\
\midrule
$4 \times 4$ conv. 64 lReLU. stride 2 & FC. 1024 ReLU. batchnorm \\
\midrule
$4 \times 4$ conv. 128 lReLU. stride 2. batchnorm & FC. $8 \times 8 \times 256$ ReLU. batchnorm \\
\midrule
$4 \times 4$ conv. 256 lReLU. stride 2. batchnorm & $4 \times 4$ upconv. 256 lReLU. stride 1. batchnorm \\
\midrule
$4 \times 4$ conv. 256 lReLU. stride 1. batchnorm & $4 \times 4$ upconv. 256 lReLU. stride 1. batchnorm \\
\midrule
$4 \times 4$ conv. 256 lReLU. stride 1. batchnorm & $4 \times 4$ upconv. 128 lReLU. stride 2. batchnorm \\
\midrule
FC. 1024 lReLU. batchnorm & $4 \times 4$ upconv. 64 lReLU. stride 2. batchnorm \\
\midrule
FC. 1. output layer for D & $4 \times 4$ upconv. 1 Sigmoid. stride 2 \\
\midrule
FC. 128 lReLU. batchnorm. FC $4 \times 2 + 1 \times 3$ for Q  & \\
\bottomrule
\end{tabular}
}
\end{small}
\end{center}

\end{table}

\begin{figure}[h!]
  \centering
  \includegraphics[width=\columnwidth]{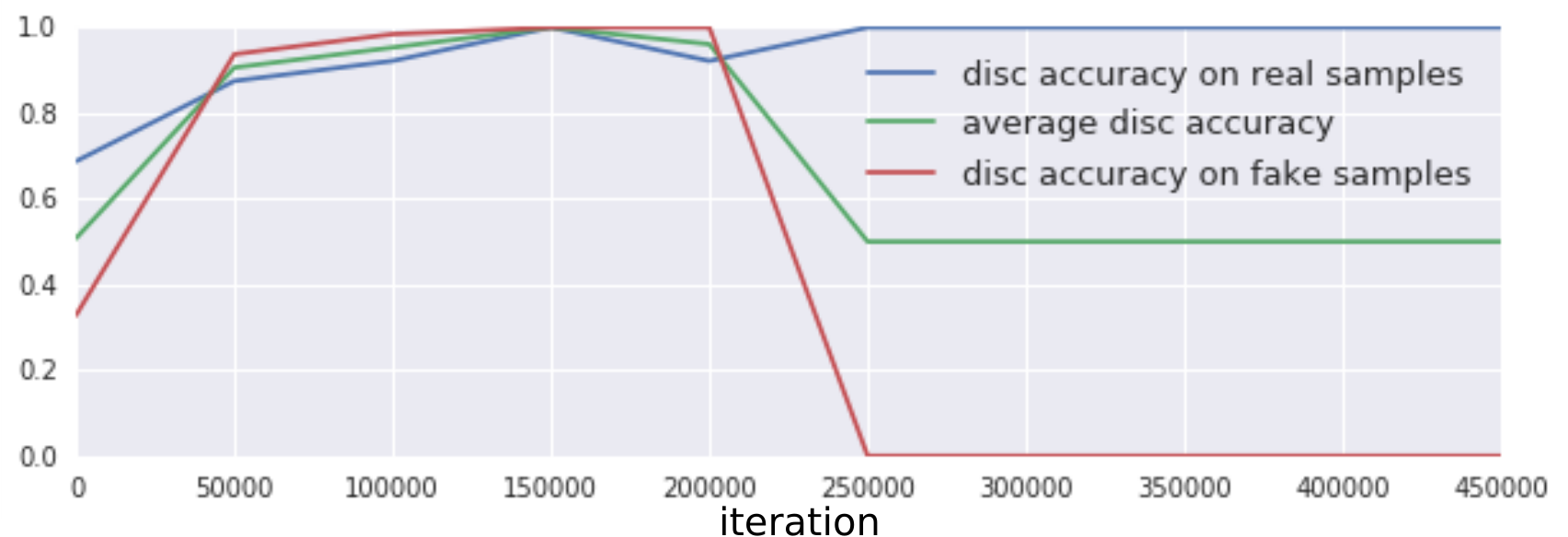}
  \vskip -0.1in
  \caption{Discriminator accuracy of InfoGAN on MNIST throughout training.}\label{fig:infogan_mnist_disc_acc}
\vskip -0.1in
\end{figure}

\begin{figure}[h!]
  \centering
  \includegraphics[width=0.7\columnwidth]{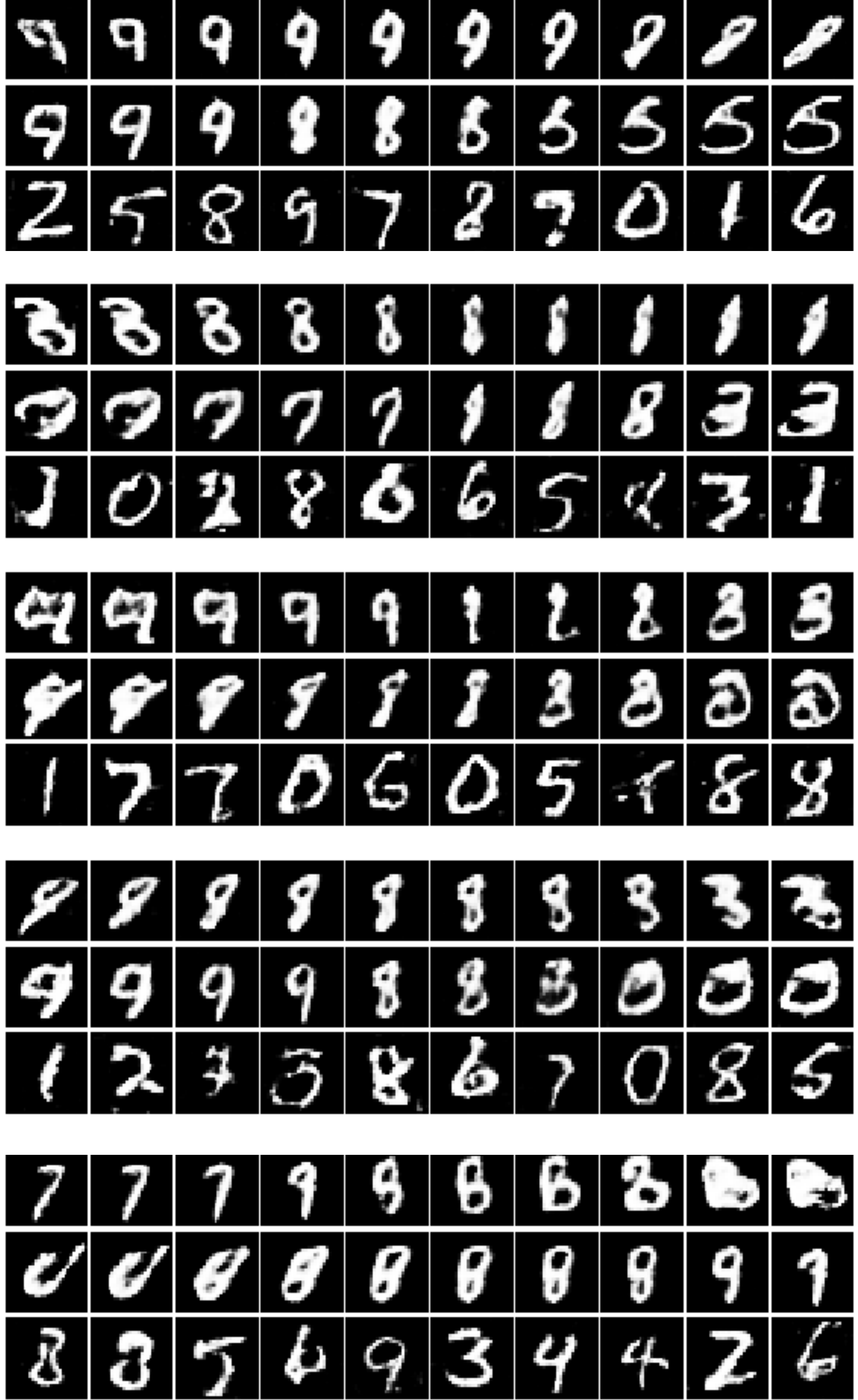}
  \vskip -0.1in
  \caption{Latent traversals for InfoGAN on MNIST across the two continuous codes (first two rows) and the categorical code (last row) for 5 different random seeds.}\label{fig:infogan_mnist_latent_traversals_multiple}
\vskip -0.1in
\end{figure}

\begin{figure}[h!]
  \centering
  \includegraphics[width=0.7\columnwidth]{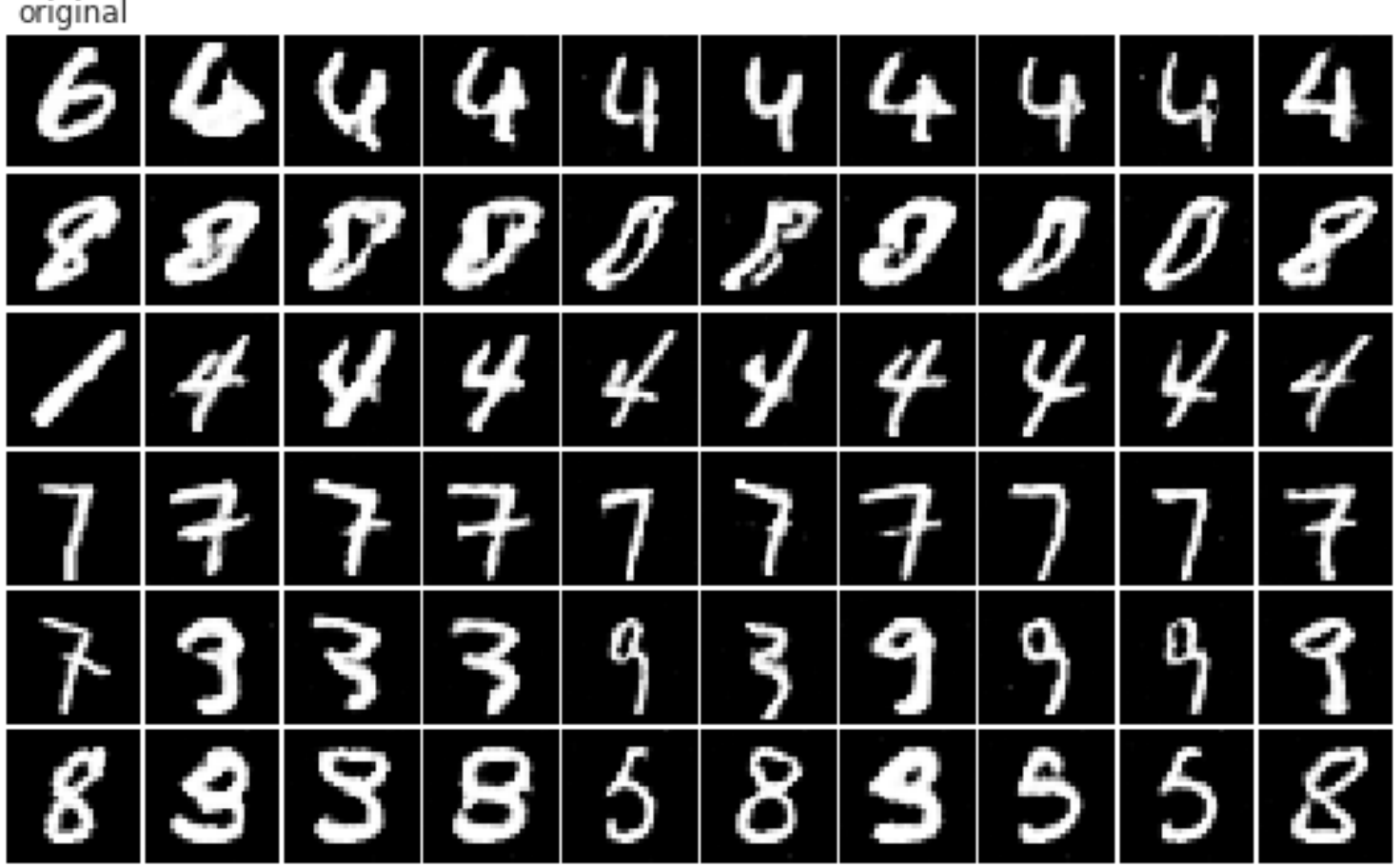}
  \vskip -0.1in
  \caption{Reconstructions for InfoGAN on MNIST. First column: original image. Remaining columns: reconstructions varying the noise latent $\epsilon$.}\label{fig:infogan_mnist_reconstructions}
\vskip -0.1in
\end{figure}

\begin{figure}[h!]
  \centering
  \includegraphics[width=0.7\columnwidth]{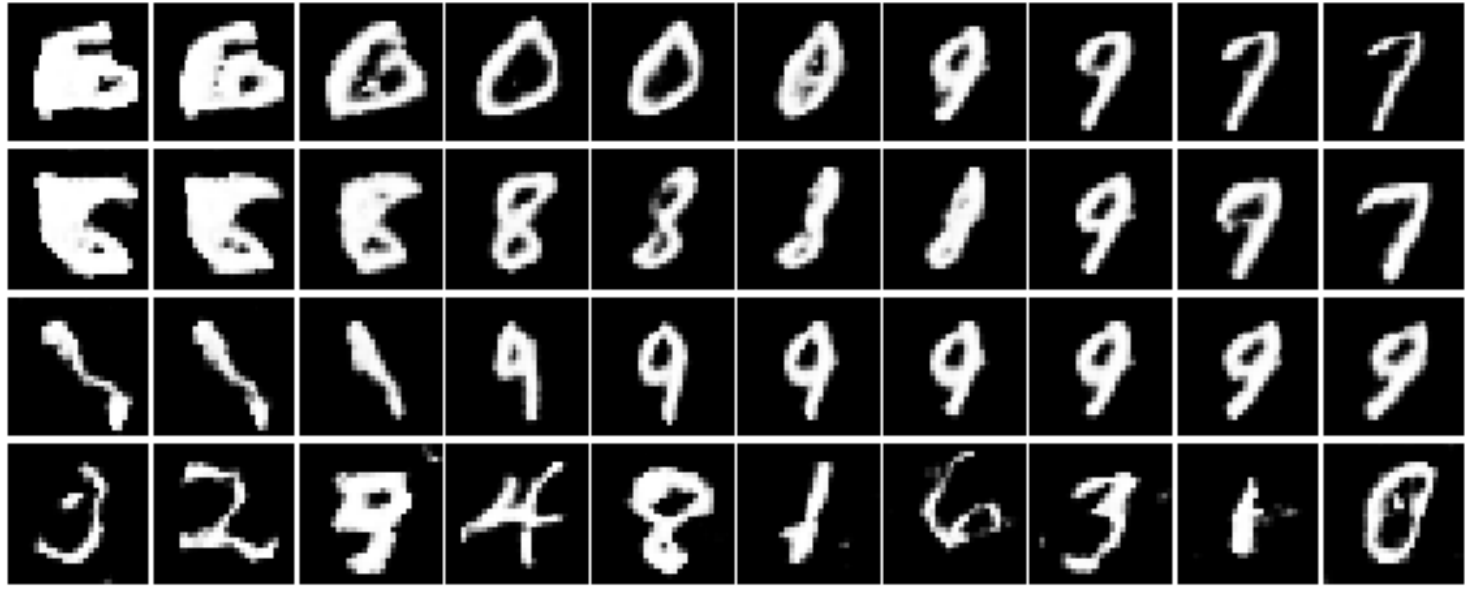}
  \vskip -0.1in
  \caption{Latent traversals for InfoGAN on MNIST across the three continuous codes (first three rows) and the categorical code (last row).}\label{fig:infogan_mnist_latent_traversals_3cont}
\vskip -0.1in
\end{figure}

We firstly observe that for all runs, we eventually get a degenerate discriminator that predicts all inputs to be real, as in Figure \ref{fig:infogan_mnist_disc_acc}. This is the well-known instability issue of the original GAN. We have tried using a smaller learning rate for the discriminator, and although this delays the degenerate behaviour it does not prevent it. Hence early stopping seems crucial, and all results shown below are from well before the degenerate behaviour occurs.

\citet{chen2016infogan} claim that the categorical code learns digit class (discrete factor of variation) and that the continuous codes learn azimuth and width, but when plotting latent traversals for each run, we observed that this is inconsistent. We show five randomly chosen runs in Figure \ref{fig:infogan_mnist_latent_traversals_multiple}. The digit class changes in the continuous code traversals and there are overlapping digits in the categorical code traversal. Similar results hold for InfoWGAN-GP in Figure \ref{fig:infowgan_gp_mnist_latent_traversals_multiple}. 

We also tried visualising the reconstructions: given an image, we push the image through the encoder to obtain latent codes $c$, fix this $c$ and vary the noise $\epsilon$ to generate multiple reconstructions for the same image. This is to check the extent to which the noise $\epsilon$ can affect the generation. We can see in Figure \ref{fig:infogan_mnist_reconstructions} that digit class often changes when varying $\epsilon$, so the model struggles to cleanly separate semantically meaningful information and incompressible noise.

Furthermore, we investigated the sensitivity of the model to the number of latent codes. We show latent traversals using three continuous codes instead of two in Figure \ref{fig:infogan_mnist_latent_traversals_3cont}. It is evident that the model tries to put more digit class information into the continuous traversals. So the number of codes is an important hyperparameter to tune, whereas VAE methods are less sensitive to the choice of number of codes since they can prune out unnecessary latents by collapsing $q(z_j|x)$ to the prior $p(z_j)$.

\begin{figure}[h!]
  \centering
  \includegraphics[width=0.7\columnwidth]{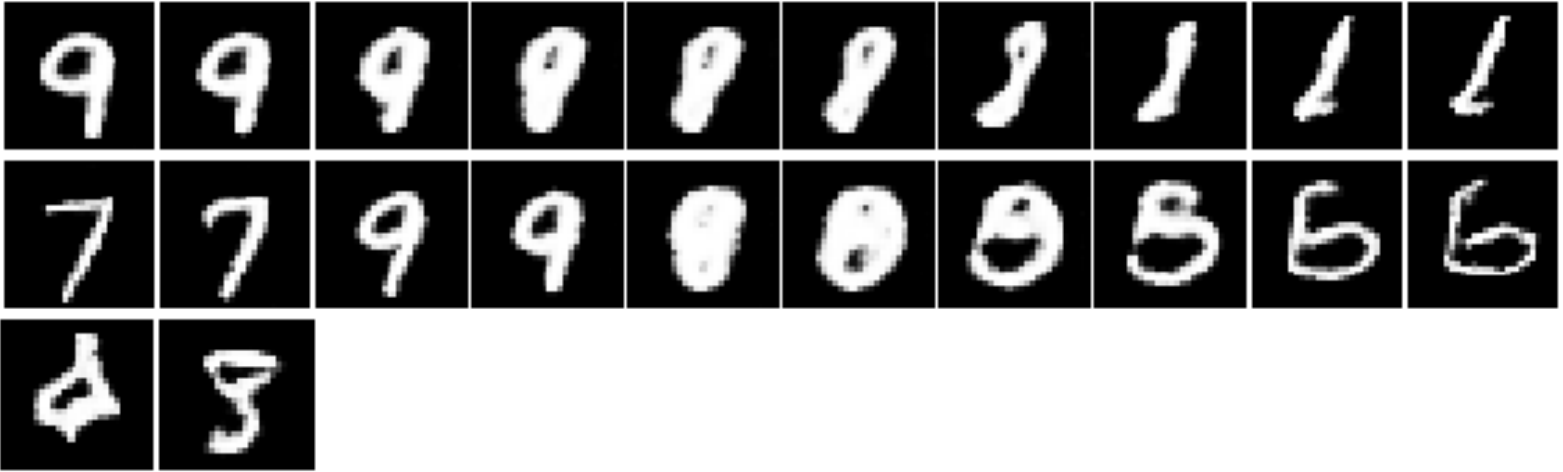}
  \vskip -0.1in
  \caption{Latent traversals for InfoGAN on MNIST across the two continuous codes (first two rows) and the categorical code (last row) using 2 categories}\label{fig:infogan_mnist_latent_traversals_2catval}
\vskip -0.1in
\end{figure}

\begin{figure}[h!]
  \centering
  \includegraphics[width=0.7\columnwidth]{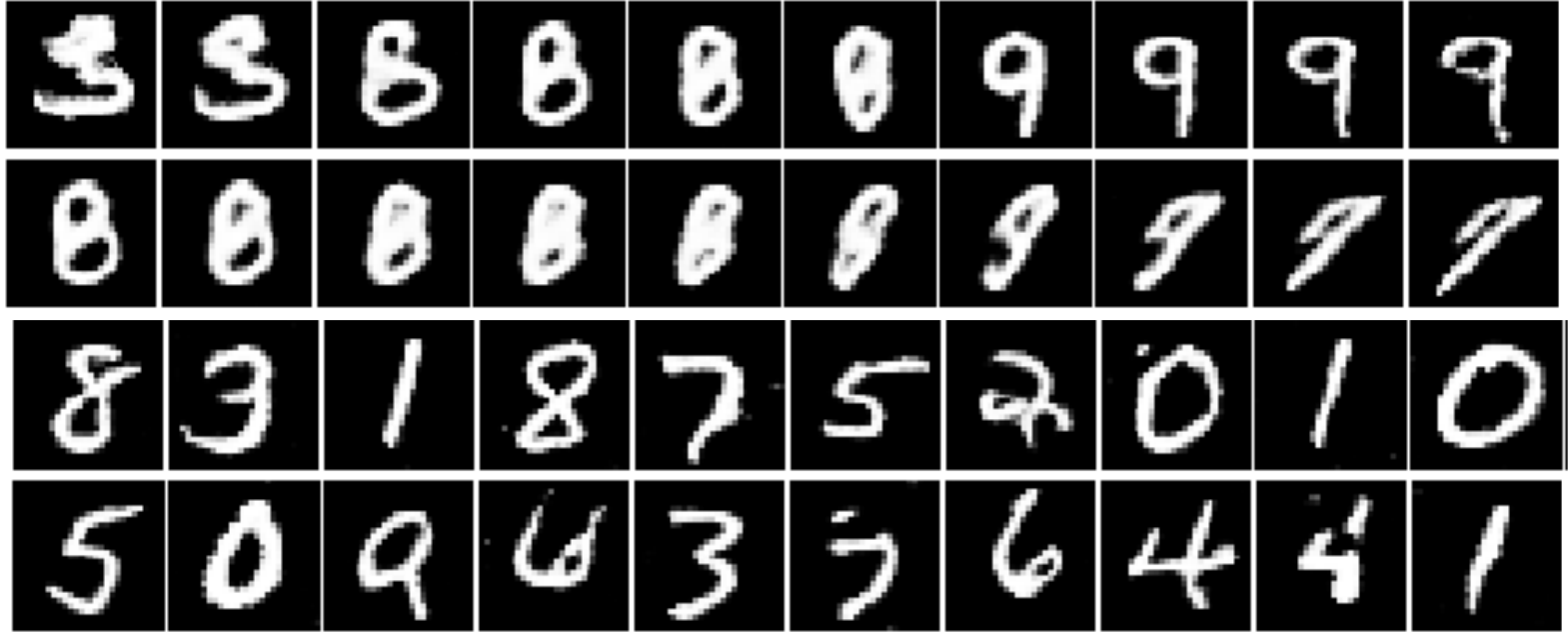}
  \vskip -0.1in
  \caption{Latent traversals for InfoGAN on MNIST across the two continuous codes (first two rows) and the categorical code (last two rows) using 20 categories}\label{fig:infogan_mnist_latent_traversals_20catval}
\vskip -0.1in
\end{figure}

\begin{figure}[h!]
  \centering
  \includegraphics[width=0.7\columnwidth]{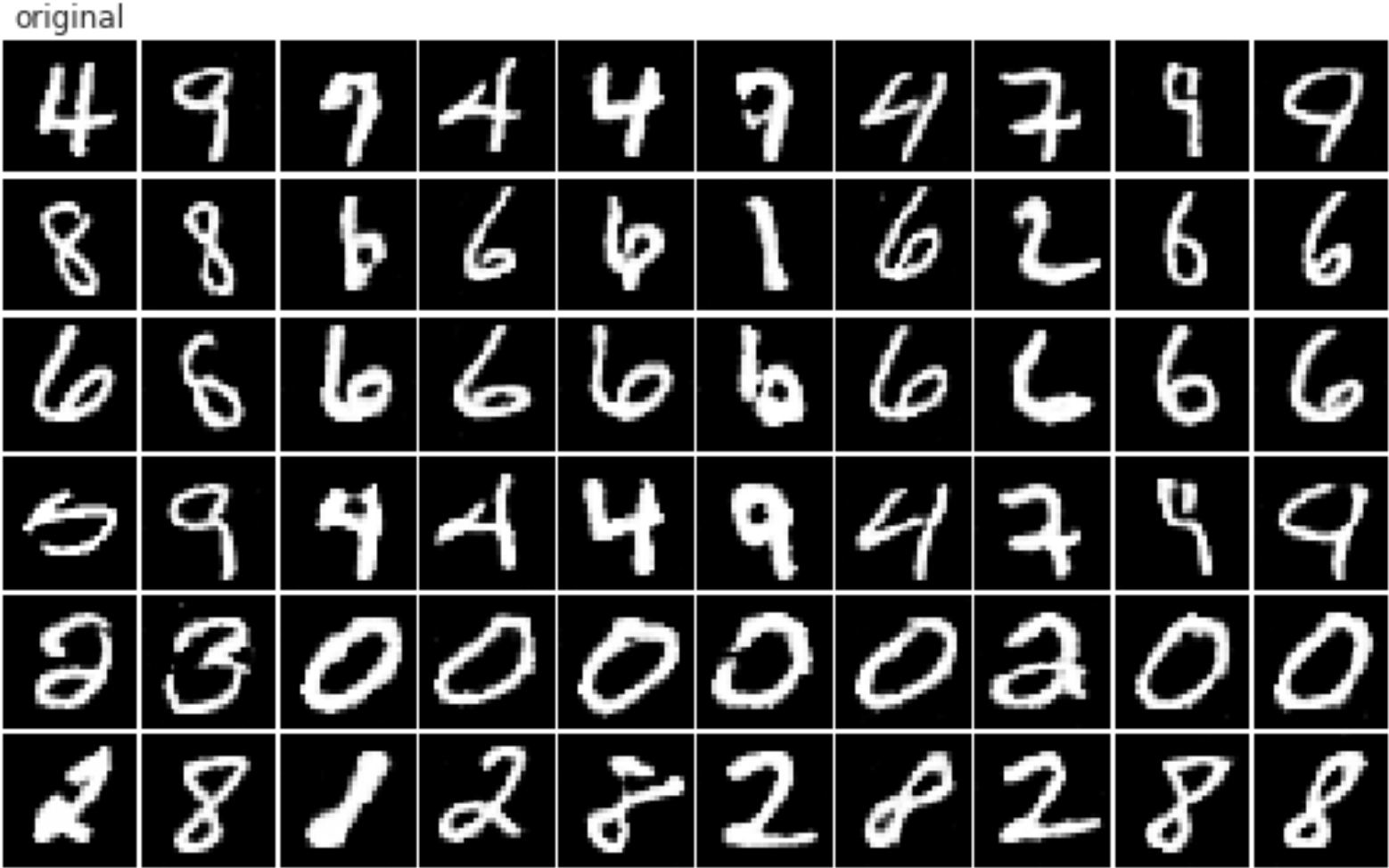}
  \vskip -0.1in
  \caption{Same as Figure \ref{fig:infogan_mnist_reconstructions} but the categorical code having 2 categories.}\label{fig:infogan_mnist_reconstructions_2catval}
\vskip -0.1in
\end{figure}

\begin{figure}[h!]
  \centering
  \includegraphics[width=0.7\columnwidth]{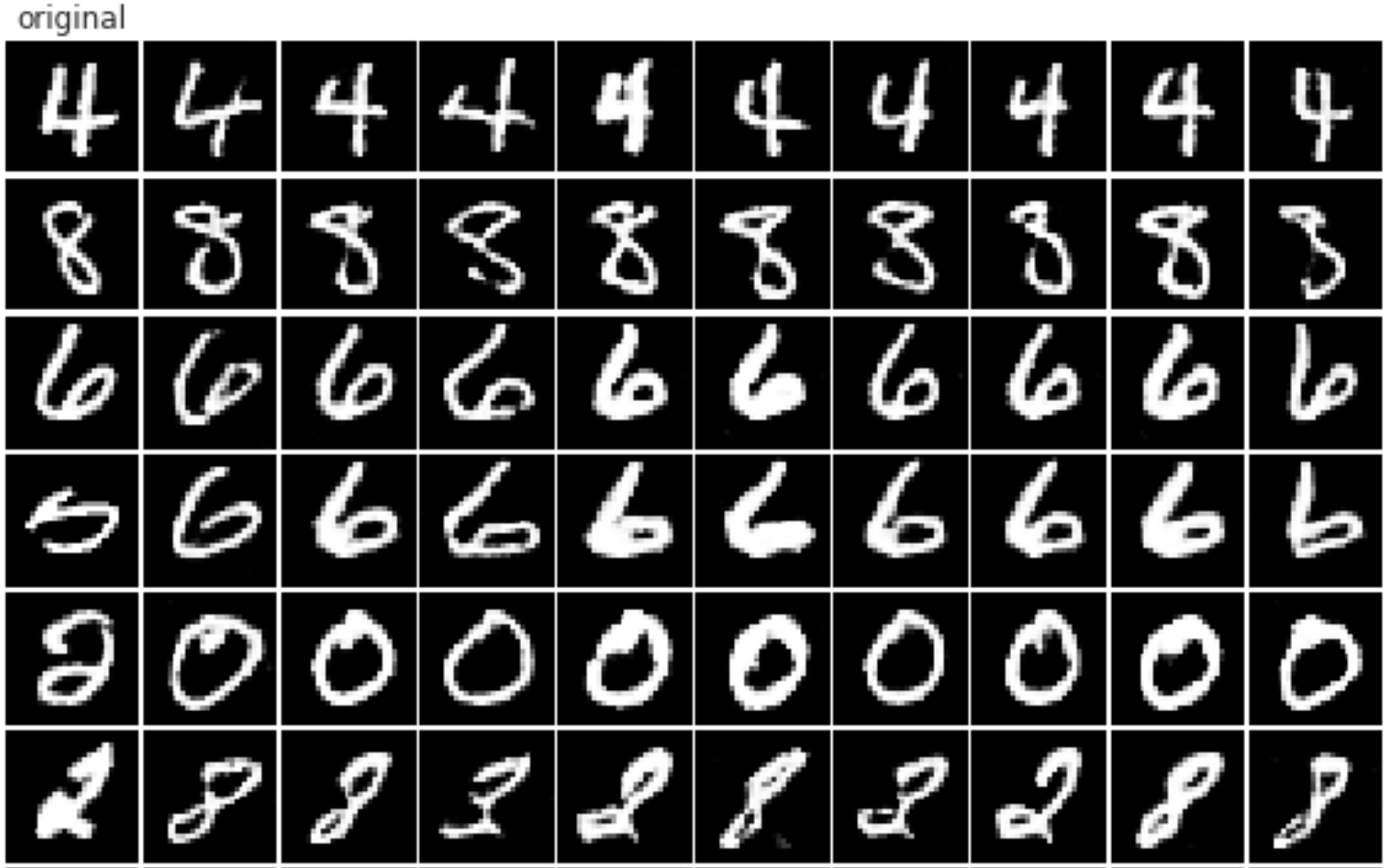}
  \vskip -0.1in
  \caption{Same as Figure \ref{fig:infogan_mnist_reconstructions_2catval} but with 20 categories.}\label{fig:infogan_mnist_reconstructions_20catval}
\vskip -0.1in
\end{figure}

We also tried varying the number of categories for the categorical code. Using 2 categories, we see from Figure \ref{fig:infogan_mnist_latent_traversals_2catval} that the model tries to put much more information about digit class into the continuous latents, as expected. Moreover from Figure \ref{fig:infogan_mnist_reconstructions_2catval}, we can see that the noise variables also have more information about the digit class. However, when we use 20 categories, we see that the model still puts information about the digit class in the continuous latents. However from Figure \ref{fig:infogan_mnist_reconstructions_20catval} we see that the noise variables contain less semantically meaningful information.

\begin{figure}[h!]
\vskip -0.1in
  \centering
  \includegraphics[width=0.7\columnwidth]{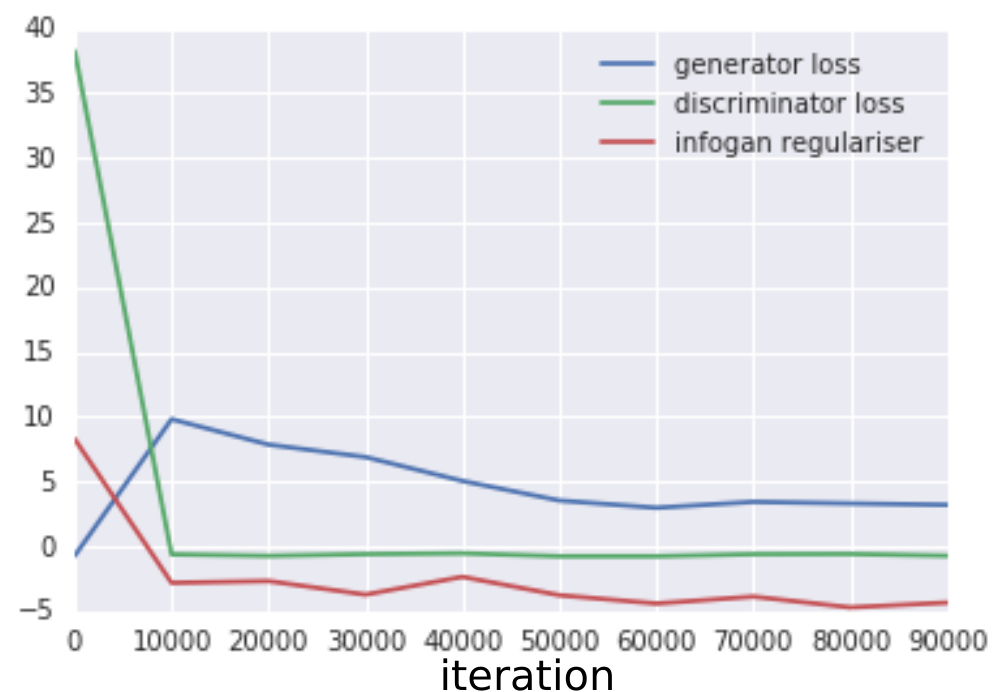}
  \vskip -0.1in
  \caption{The generator loss $- \mathbb{E}_{p(z)}[D(G(z))]$, discriminator loss $\mathbb{E}_{p_{data}(x)}[D(x)] - \mathbb{E}_{p(z)}[D(G(z))]$ and the InfoGAN regulariser term $-L$ for InfoWGAN-GP on MNIST with $\lambda=1$}\label{fig:infowgan_gp_mnist_learning_curve}
\vskip -0.1in
\end{figure}

\begin{figure}[h!]
  \centering
  \includegraphics[width=0.7\columnwidth]{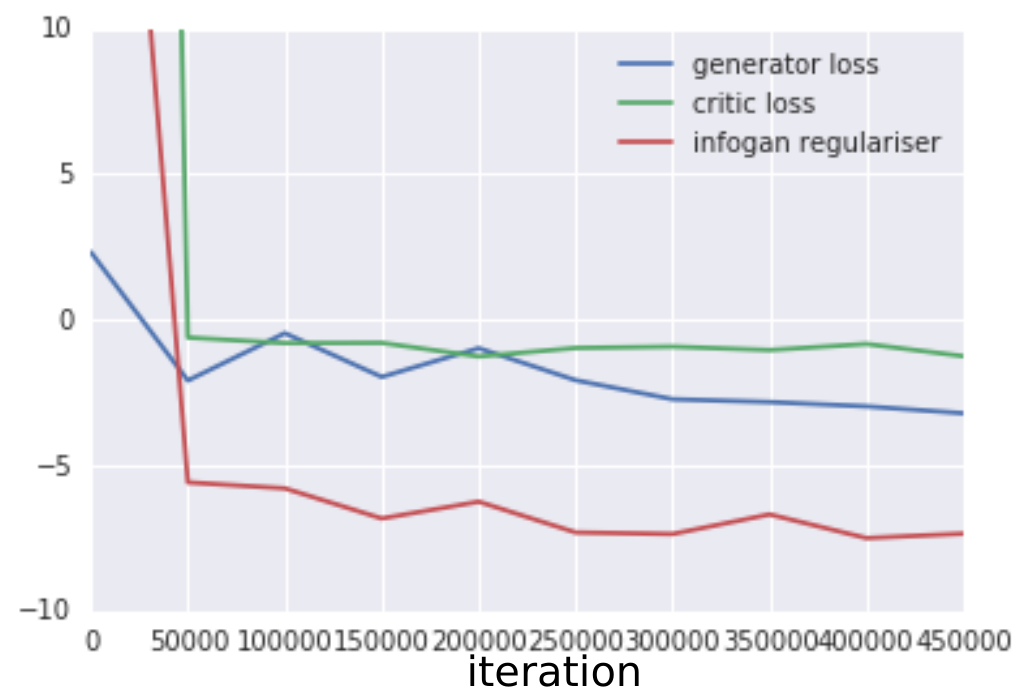}
  \vskip -0.1in
  \caption{Same as Figure \ref{fig:infowgan_gp_mnist_learning_curve} but for 2D Shapes.}\label{fig:infowgan_gp_shapes2d_learning_curve}
\vskip -0.1in
\end{figure}

\begin{figure}[h!]
  \centering
  \includegraphics[width=\columnwidth]{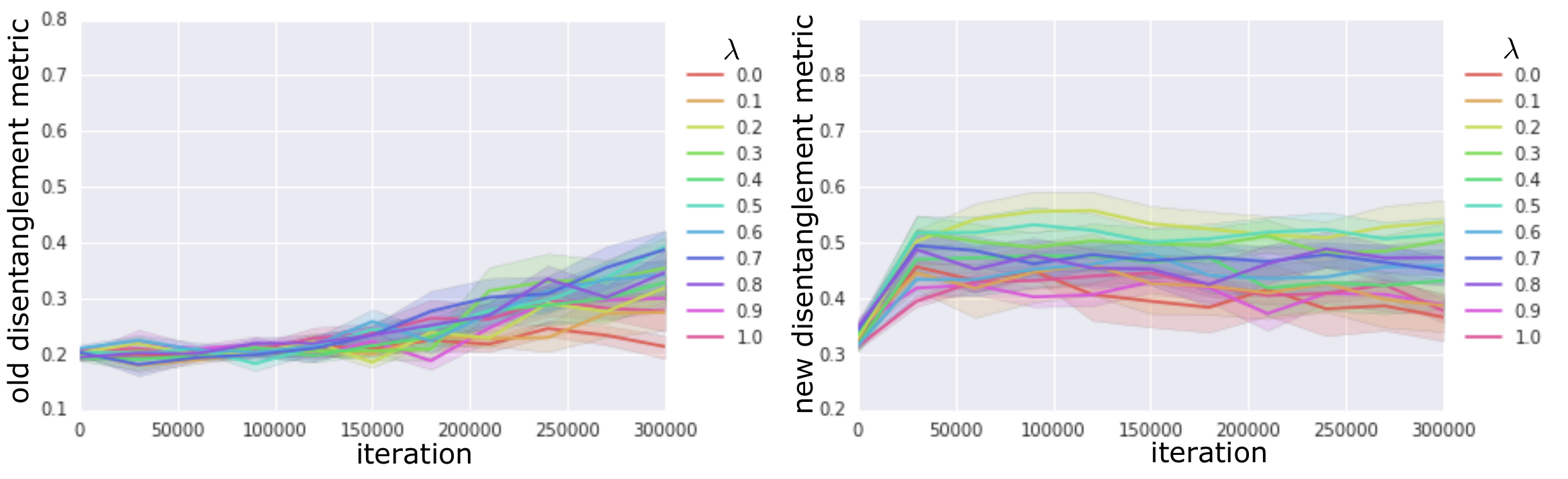}
  \vskip -0.1in
  \caption{Disentanglement scores for InfoWGAN-GP on 2D Shapes with bigger architecture (Table \ref{tab:2dshapes_infogan_big}) for 10 random seeds per hyperparameter setting. Left: Metric in \citet{higgins2016beta}. Right: Our metric.}\label{fig:infowgan_gp_shapes2d_disent_big}
\vskip -0.1in
\end{figure}

\begin{figure}[h!]
  \centering
  \includegraphics[width=0.7\columnwidth]{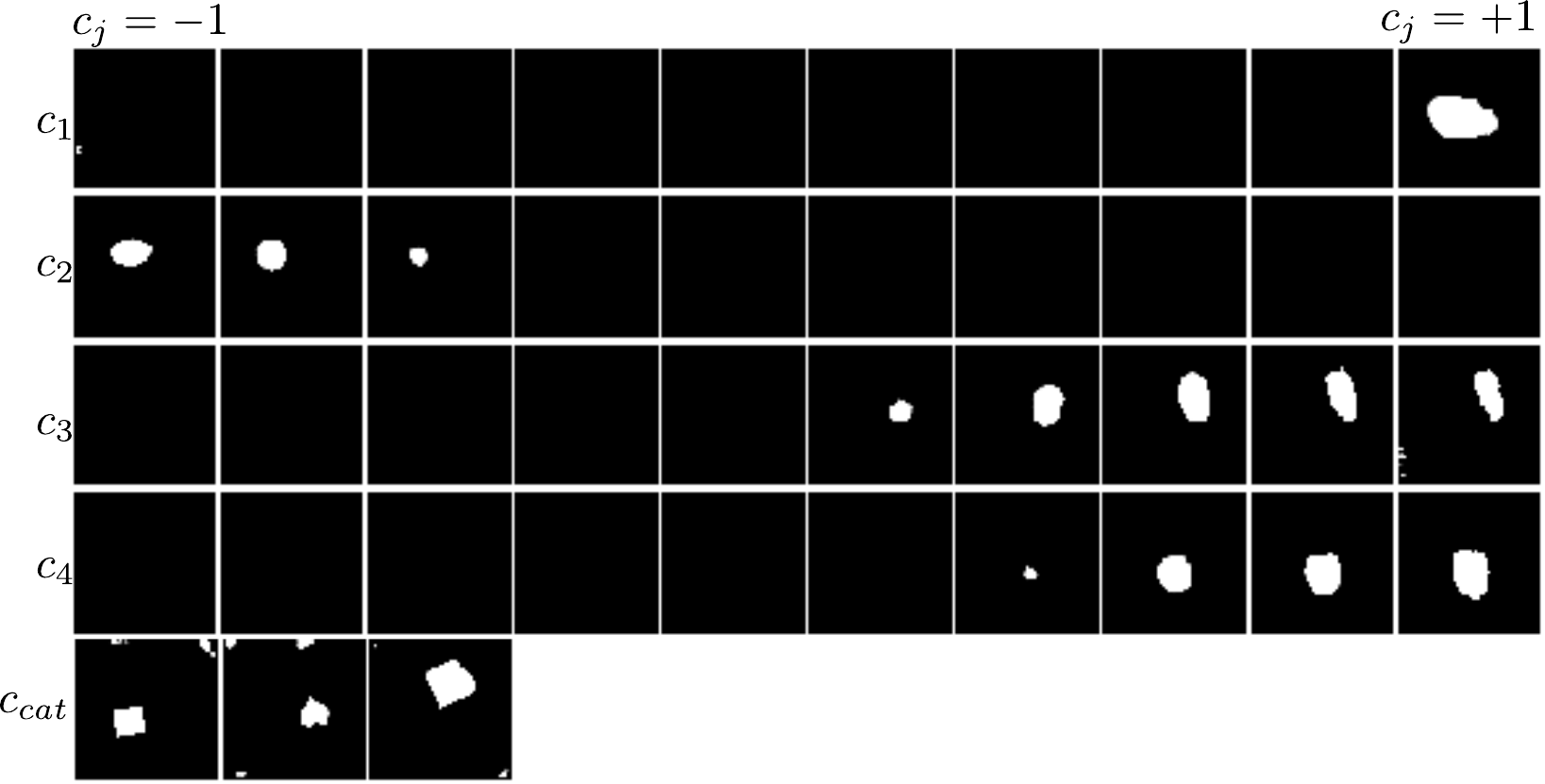}
  \vskip -0.1in
  \caption{Latent traversals for InfoWGAN-GP on 2D Shapes across the four continuous codes (first four rows) and the categorical code (last row) with bigger architecture (Table \ref{tab:2dshapes_infogan_big}) for run with best disentanglement score ($\lambda=0.6$).}\label{fig:infowgan_gp_shapes2d_latent_traversals_big}
\vskip -0.2in
\end{figure}

Using InfoWGAN-GP solved the degeneracy issue and makes training more stable (see Figure \ref{fig:infowgan_gp_shapes2d_learning_curve}), but we observed that the other problems persisted (see \eg  Figure \ref{fig:infowgan_gp_mnist_latent_traversals_multiple}).

For 2D Shapes, we have also tried using a bigger architecture for InfoWGAN-GP that is used for a data set of similar dimensions (Chairs data set) in \citet{chen2016infogan}. See Table \ref{tab:2dshapes_infogan_big}. However as can be seen in Figure \ref{fig:infowgan_gp_shapes2d_disent_big} this did not improve disentanglement scores, yet the latent traversals look slightly more realistic (Figure \ref{fig:infowgan_gp_shapes2d_latent_traversals_big}).

In summary, InfoWGAN-GP can help prevent the instabilities in training faced by InfoGAN, but it does not help overcome the following weaknesses compared to VAE-based methods: 1) Disentangling performance is sensitive to the number of code latents. 2) More often than not, the noise variables contain semantically meaningful information. 3) The model does not always generalise well to all across the domain of $p(z)$.

\begin{figure}[h!]
  \centering
  \includegraphics[width=0.7\columnwidth]{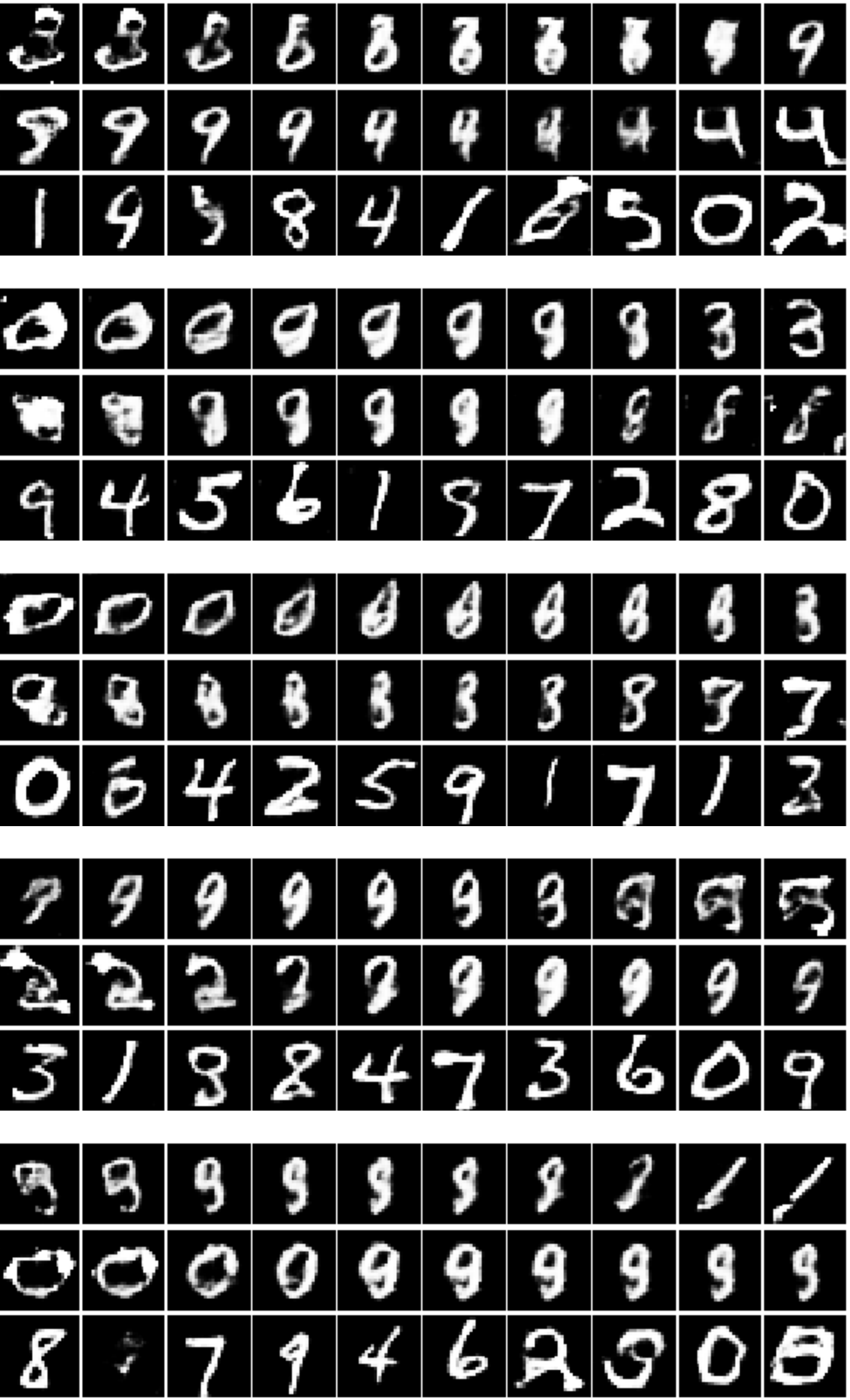}
  \vskip -0.1in
  \caption{Same as Figure \ref{fig:infogan_mnist_latent_traversals_multiple} but for InfoWGAN-GP.}\label{fig:infowgan_gp_mnist_latent_traversals_multiple}
\vskip -0.1in
\end{figure}

\section{Further Experimental Results}
\label{apd:further_results}

From Figure \ref{fig:shapes2d_disc_acc}, we see that higher values of $\gamma$ in FactorVAE leads to a lower discriminator accuracy. This is as expected, since a higher $\gamma$ encourages $q(z)$ and $\prod_j q(z_j)$ to be closer together, hence a lower accuracy for the discriminator to successfully classify samples from the two distributions.

\begin{figure}[h!]
  \centering
  \includegraphics[width=\columnwidth]{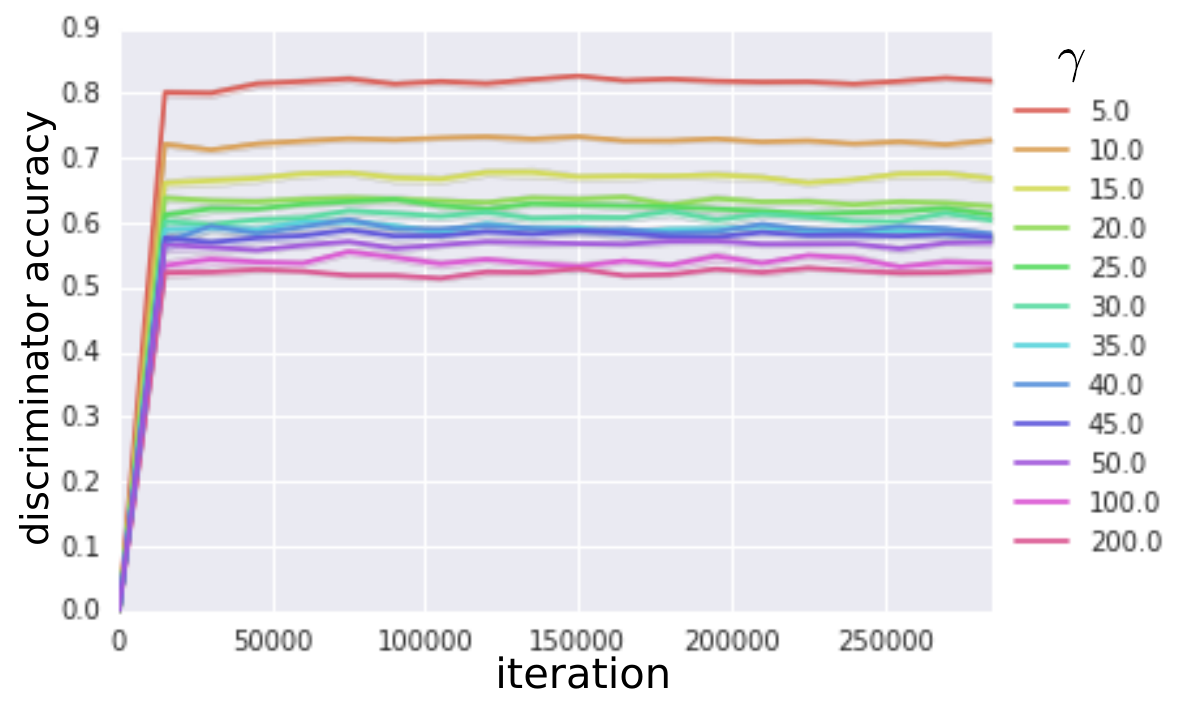}
  \vskip -0.1in
  \caption{Plot of discriminator accuracy of FactorVAE on 2D Shapes data across iterations over 5 random seeds.}\label{fig:shapes2d_disc_acc}
\vskip -0.1in
\end{figure}

\begin{figure}[htb!]
  \centering
  \includegraphics[width=\columnwidth]{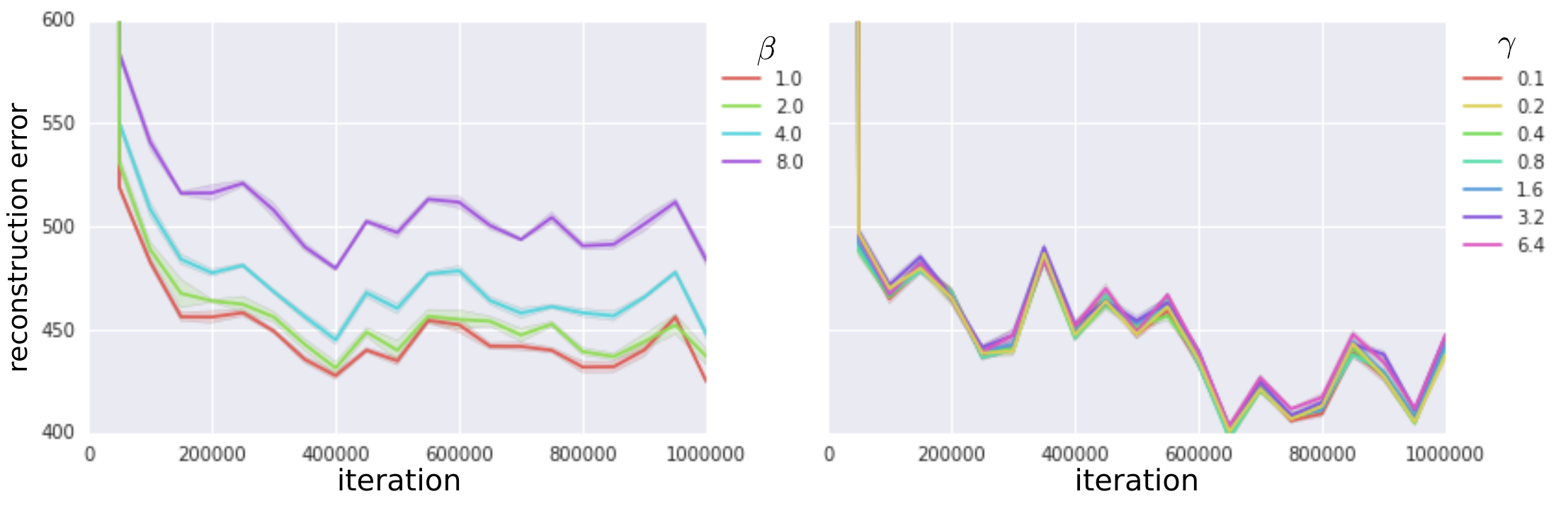}
  \vskip -0.1in
  \caption{Same as \fig{fig:faces3d_rec_err} but for 3D Chairs.}\label{fig:chairs_rec_err}
\vskip -0.1in
\end{figure}

\begin{figure}[htb!]
  \centering
  \includegraphics[width=\columnwidth]{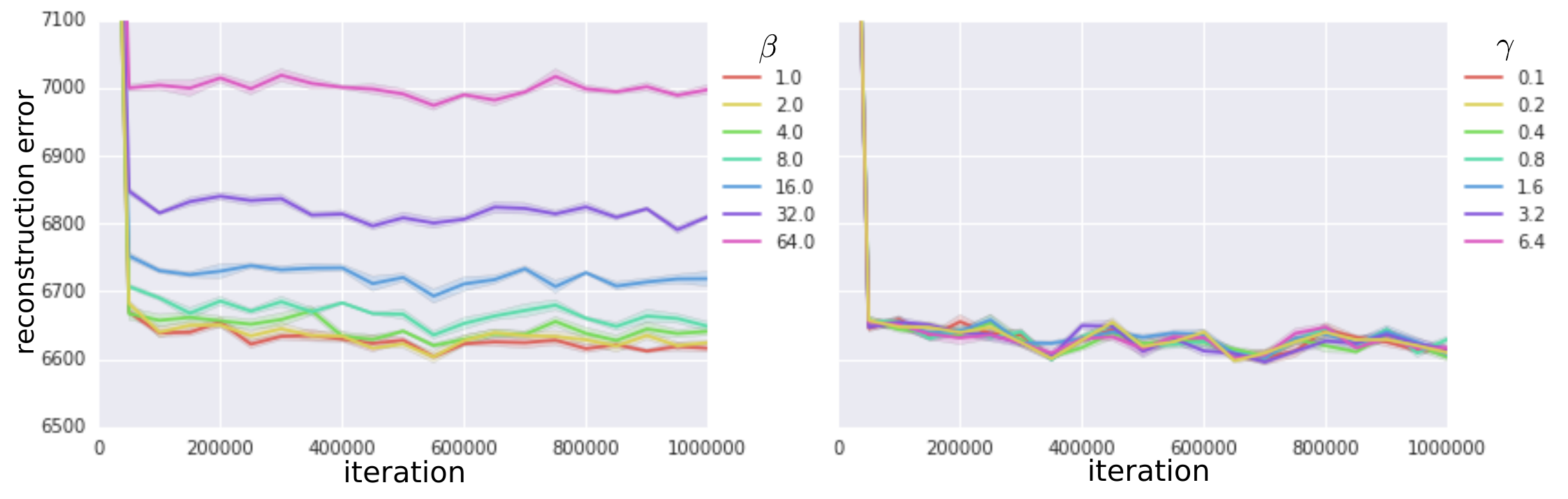}
  \vskip -0.1in
  \caption{Same as \fig{fig:faces3d_rec_err} but for CelebA.}\label{fig:celeba_rec_err}
\vskip -0.1in
\end{figure}

\begin{figure*}[t!]
  \centering
  \includegraphics[width=\linewidth]{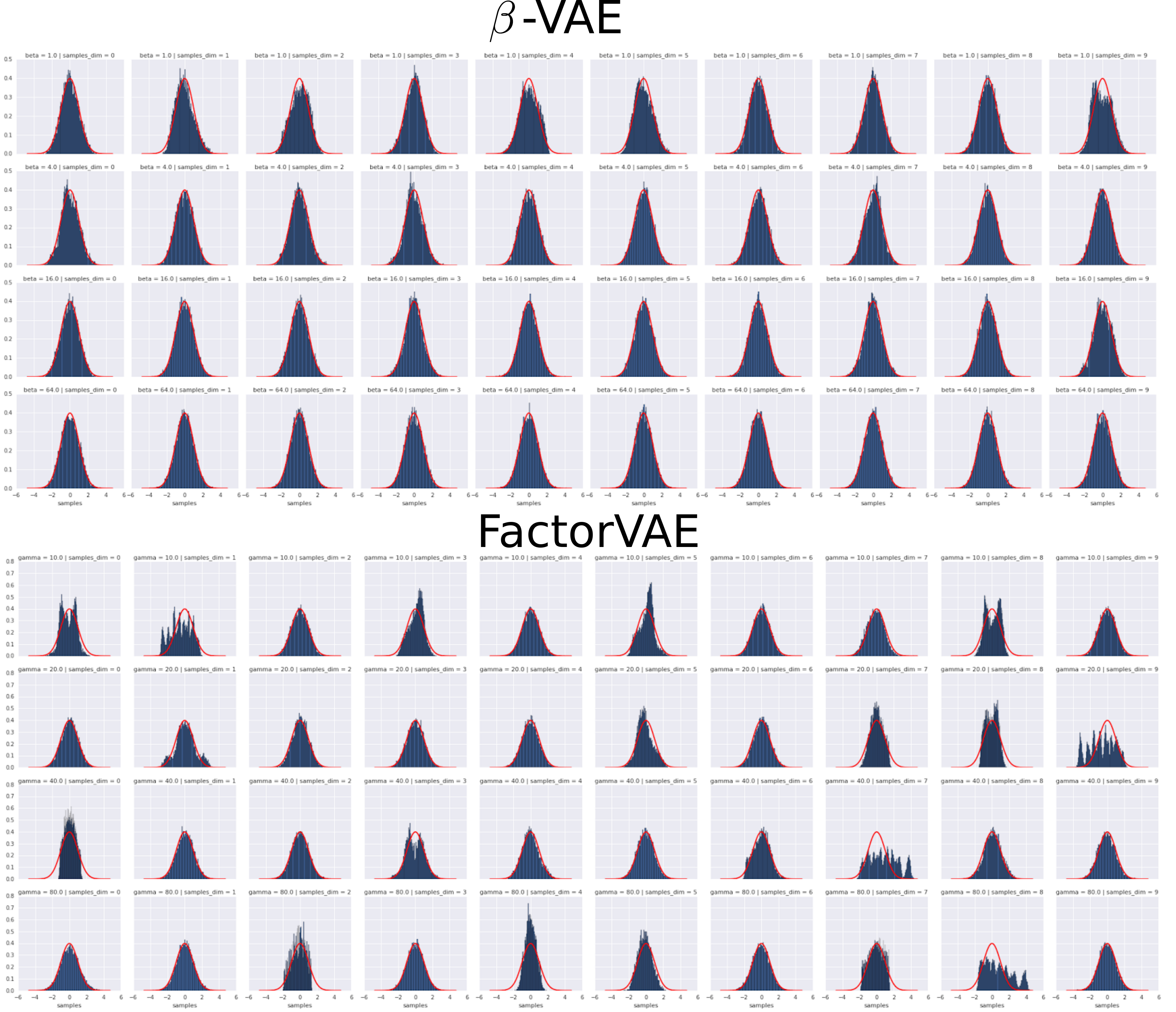}
  \vskip -0.1in
  \caption{Histograms of $q(z_j)$ for each $j$ (columns) for $\beta$-VAE and FactorVAE at the end of training on 2D Shapes, with the pdf of Gaussian $\mathcal{N}(0,1)$ overlaid in red. The rows correspond to different values of $\beta$ $(1,4,16,64)$ and $\gamma$ $(10,20,40,80)$ respectively.}\label{fig:shapes2d_histograms}
\vskip -0.1in
\end{figure*}

We also show histograms of $q(z_j)$ for each $j$ in $\beta$-VAE and FactorVAE for different values of $\beta$ and $\gamma$ at the end of training on 2D Shapes in Figure \ref{fig:shapes2d_histograms}. We can see that the marginals of FactorVAE are quite different from the prior, which could be a reason that the variant of FactorVAE using the objective given by \eqn{eq:modified_factor_vae} leads to different results to FactorVAE. For FactorVAE, the model is able to focus on factorising $q(z)$ instead of pushing it towards some arbitrarily specified prior $p(z)$.

\end{document}